\documentclass{article}

\usepackage{times}
\usepackage{graphicx} %
\usepackage{subfigure} 

\usepackage{natbib}

\usepackage{algorithm}
\usepackage{algorithmic}

\usepackage{hyperref}

\usepackage{wrapfig}

\usepackage{xcolor}
\usepackage[accepted]{icml2017}

\icmltitlerunning{Reinforcement Learning with Deep Energy-Based Policies}

\usepackage{amsfonts}
\usepackage{amsopn}
\usepackage{amsmath}
\usepackage{amssymb}
\usepackage{mathtools}
\usepackage{amsthm}
\usepackage{placeins}
\usepackage{appendix}

\newtheorem{theorem}{Theorem}

\newcommand{\eg}{e.g.\ }

\newcommand{\E}[2]{\operatorname{\mathbb{E}}_{#1}\left[#2\right]}
\newcommand{\EEE}{\mathbb{E}}
\newcommand{\density}{p}
\newcommand{\kl}[2]{\mathrm{D_{KL}}\left(#1\;\middle\|\;#2\right)}
\newcommand{\entropy}{\mathcal{H}}
\newcommand{\ent}{\mathcal{H}}
\newcommand{\sdots}{\,\cdot\,}

\newcommand{\eye}{\boldsymbol{I}}
\newcommand{\zeros}{\boldsymbol{0}}

\newcommand{\sspace}{\mathcal{S}}
\newcommand{\aspace}{\mathcal{A}}
\newcommand{\state}{\mathbf{s}}

\newcommand{\st}{{\state_t}}
\newcommand{\sti}{{\state_t^{(i)}}}

\newcommand{\stp}{{\state_{t+1}}}
\newcommand{\stpi}{{\state_{t+1}^{(i)}}}
\newcommand{\pdyn}{\density_\state}

\newcommand{\action}{\mathbf{a}}

\newcommand{\at}{{\action_t}}
\newcommand{\ati}{{\action_t^{(i)}}}

\newcommand{\attildej}{{\tilde{\action}_t^{(j)}}}
\newcommand{\atij}{{\action_t^{(i,j)}}}

\newcommand{\opt}{^*}

\newcommand{\reward}{r}
\newcommand{\rz}{\reward_0}
\newcommand{\rt}{\reward_t}
\newcommand{\rti}{\reward^{(i)}_t}
\newcommand{\rmin}{r_\mathrm{min}}
\newcommand{\rmax}{r_\mathrm{max}}

\newcommand{\Vsoft}{V_\mathrm{soft}}
\newcommand{\Vsoftparams}{V_\mathrm{soft}^\qparams}

\newcommand{\Q}{Q}
\newcommand{\Qsoft}{Q_\mathrm{soft}}
\newcommand{\Qsoftparams}{Q_\mathrm{soft}^\qparams}
\newcommand{\Qhatsoft}{\hat Q_\mathrm{soft}}
\newcommand{\Qhatsoftparams}{\hat Q_\mathrm{soft}^{\bar\qparams}}

\newcommand{\Ahatsoft}{\hat A_\mathrm{soft}}

\newcommand{\energy}{\mathcal{E}}

\newcommand{\policy}{\pi}

\newcommand{\params}{\theta}

\newcommand{\pparams}{{\phi}}   % policy parameters
\newcommand{\qparams}{{\theta}}   % q function parameters
\newcommand{\qtargetparams}{{\bar\theta}}   % q function parameters
\newcommand{\pgparams}{\phi}

\newcommand{\pgloss}{J_\mathrm{PG}}
\newcommand{\qloss}{J_\Q}  % Bellman error loss (Q function loss).
\newcommand{\ploss}{J_\policy}  % Stein error loss (policy loss).

\newcommand{\temp}{\alpha}  % temperature

\newcommand{\kernel}{\kappa}
\newcommand{\smpl}{\mathbf{\xi}}  % generic sample
\newcommand{\smpli}{\smpl^{(i)}}  % generic sample
\newcommand{\smpltildej}{\tilde{\xi}^{(j)}}  % alternative generic sample
\newcommand{\smplij}{\xi^{(i,j)}}  % generic sample

\newcommand{\gauss}{\mathcal{N}}

\newcommand{\discount}{\gamma}

\newcommand{\aref}[1]{\hyperref[#1]{Appendix~\ref{#1}}}

\def\alignautorefname~#1\null{%
  (#1)\null
}\def\equationautorefname~#1\null{%
  (#1)\null
}
\begin{document} 
\twocolumn[
\icmltitle{Reinforcement Learning with Deep Energy-Based Policies}

\icmlsetsymbol{equal}{*}

\begin{icmlauthorlist}
\icmlauthor{Tuomas Haarnoja}{equal,ucb}
\icmlauthor{Haoran Tang}{equal,ucbmath}
\icmlauthor{Pieter Abbeel}{ucb,oai,icsi}
\icmlauthor{Sergey Levine}{ucb}
\end{icmlauthorlist}

\icmlaffiliation{ucb}{UC Berkeley, Department of Electrical Engineering and Computer Sciences}
\icmlaffiliation{ucbmath}{UC Berkeley, Department of Mathematics}
\icmlaffiliation{oai}{OpenAI}
\icmlaffiliation{icsi}{International Computer Science Institute}

\icmlcorrespondingauthor{Haoran Tang}{hrtang@math.berkeley.edu}
\icmlcorrespondingauthor{Tuomas Haarnoja}{haarnoja@berkeley.edu}

\icmlkeywords{reinforcement learning, deep learning, maximum entropy, energy-based models}

\vskip 0.3in
]

\printAffiliationsAndNotice{\icmlEqualContribution} %

\begin{abstract} 
We propose a method for learning expressive energy-based policies for continuous states and actions, which has been feasible only in tabular domains before. We apply our method to learning maximum entropy policies, resulting into a new algorithm, called soft Q-learning, that expresses the optimal policy via a Boltzmann distribution. We use the recently proposed amortized Stein variational gradient descent to learn a stochastic sampling network that approximates samples from this distribution. The benefits of the proposed algorithm include improved exploration and compositionality that allows transferring skills between tasks, which we confirm in simulated experiments with swimming and walking robots. We also draw a connection to actor-critic methods, which can be viewed performing approximate inference on the corresponding energy-based model.
\end{abstract}

\vspace{-8mm}
\section{Introduction}
\label{sec:introduction}
\vspace{-1mm}

Deep reinforcement learning (deep RL) has emerged as a promising direction for autonomous acquisition of complex behaviors \cite{mnih2015human,silver2016mastering}, due to its ability to process complex sensory input \cite{jaderberg2016reinforcement} and to acquire elaborate behavior skills using general-purpose neural network representations \cite{levine2016end}. Deep reinforcement learning methods can be used to optimize deterministic \cite{lillicrap2015continuous} and stochastic \cite{schulman2015trust, mnih2016asynchronous} policies. However, most deep RL methods operate on the conventional deterministic notion of optimality, where the optimal solution, at least under full observability, is always a deterministic policy \cite{sutton1998reinforcement}. Although stochastic policies are desirable for exploration, this exploration is typically attained heuristically, for example by injecting noise \cite{silver2014deterministic, lillicrap2015continuous, mnih2015human} or initializing a stochastic policy with high entropy \cite{kakade2002natural, schulman2015trust, mnih2016asynchronous}.

In some cases, we might actually prefer to learn stochastic behaviors. In this paper, we explore two potential reasons for this: exploration in the presence of multimodal objectives, and compositionality 
attained via pretraining. Other benefits include robustness in the face of uncertain dynamics \cite{ziebart2010modeling}, imitation learning \cite{ziebart2008maximum}, and improved convergence and computational properties \cite{gu2016q}. Multi-modality also has application in real robot tasks, as demonstrated in \cite{daniel2012hierarchical}. However, in order to learn such policies, we must define an objective that promotes stochasticity. 

In which cases is a stochastic policy actually the optimal solution? As discussed in prior work, a stochastic policy emerges as the optimal answer when we consider the connection between optimal control and probabilistic inference \cite{todorov2008general}. While there are multiple instantiations of this framework, they typically include the cost or reward function as an additional factor in a factor graph, and infer the optimal conditional distribution over actions conditioned on states. The solution can be shown to optimize an entropy-augmented reinforcement learning objective or to correspond to the solution to a maximum entropy learning problem \cite{toussaint2009robot}.
Intuitively, framing control as inference produces policies that aim to capture not only the single deterministic behavior that has the lowest cost, but the entire range of low-cost behaviors, explicitly maximizing the entropy of the corresponding policy. Instead of learning the best way to perform the task, the resulting policies try to learn \emph{all} of the ways of performing the task. It should now be apparent why such policies might be preferred: if we can learn all of the ways that a given task might be performed, the resulting policy can serve as a good initialization for finetuning to a more specific behavior (e.g. first learning all the ways that a robot could move forward, and then using this as an initialization to learn separate running and bounding skills); a better exploration mechanism for seeking out the best mode in a multi-modal reward landscape; and a more robust behavior in the face of adversarial perturbations, where the ability to perform the same task in multiple different ways can provide the agent with more options to recover from perturbations.

Unfortunately, solving such maximum entropy stochastic policy learning problems in the general case is challenging. A number of methods have been proposed, including Z-learning \cite{todorov2006linearly}, maximum entropy inverse RL \cite{ziebart2008maximum}, approximate inference using message passing \cite{toussaint2009robot}, $\Psi$-learning \cite{rawlik2012stochastic}, and G-learning \cite{fox2015taming}, as well as more recent proposals in deep RL such as PGQ~\cite{o2016pgq},
but these generally operate either on simple tabular representations, which are difficult to apply to continuous or high-dimensional domains, or employ a simple parametric representation of the policy distribution, such as a conditional Gaussian. Therefore, although the policy is optimized to perform the desired skill in many different ways,
the resulting distribution is typically very limited in terms of its representational power, even if the \emph{parameters} of that distribution
are represented by an expressive function approximator, such as a neural network.

How can we extend the framework of maximum entropy policy search to arbitrary policy distributions? 
In this paper, we borrow an idea from energy-based models, which in turn reveals an intriguing connection between Q-learning, actor-critic algorithms, and probabilistic inference. In our method, we formulate a stochastic policy as a (conditional) energy-based model (EBM), with the energy function corresponding to the ``soft'' Q-function obtained when optimizing the maximum entropy objective. In high-dimensional continuous spaces, sampling from this policy, just as with any general EBM, becomes intractable. We borrow from the recent literature on EBMs to devise an approximate sampling procedure based on training a separate sampling network, which is optimized to produce unbiased samples from the policy EBM. This sampling network can then be used both for updating the EBM and for action selection. In the parlance of reinforcement learning, the sampling network is the actor in an actor-critic algorithm.
This reveals an intriguing connection: entropy regularized actor-critic algorithms can be viewed as approximate Q-learning methods, with the actor serving the role of an approximate sampler from an intractable posterior. We explore this connection further in the paper, and in the course of this discuss connections to popular deep RL methods such as deterministic policy gradient (DPG)~\cite{silver2014deterministic, lillicrap2015continuous}, 
normalized advantage functions (NAF)~\cite{gu2016continuous}, and PGQ~\cite{o2016pgq}.

The principal contribution of this work is a tractable, efficient algorithm for optimizing arbitrary multimodal stochastic policies represented by energy-based models, as well as a discussion that relates this method to other recent algorithms in RL and probabilistic inference. In our experimental evaluation, we explore two potential applications of our approach. First, we demonstrate improved exploration performance in tasks with multi-modal reward landscapes, where conventional deterministic or unimodal methods are at high risk of falling into suboptimal local optima. Second, we explore how our method can be used to provide a degree of compositionality in reinforcement learning by showing that stochastic energy-based policies can serve as a much better initialization for learning new skills than either random policies or policies pretrained with conventional maximum reward objectives.

\vspace{-2mm}
\section{Preliminaries}
\label{sec:preliminaries}
\vspace{-1mm}

In this section, we will define the reinforcement learning problem that we are addressing and briefly summarize the maximum entropy policy search objective. We will also present a few useful identities that we will build on in our algorithm, which will be presented in \autoref{sec:training}.

\vspace{-2mm}
\subsection{Maximum Entropy Reinforcement Learning}
\label{sec:maxent}

We will address learning of maximum entropy policies with approximate inference for reinforcement learning in continuous action spaces. Our reinforcement learning problem can be defined as policy search in an infinite-horizon Markov decision process (MDP), which consists of the tuple
$(\sspace, \aspace, \pdyn, \reward)$,
The state space $\sspace$ and action space $\aspace$ are assumed to be continuous, and the state transition probability $\pdyn:\ \sspace \times \sspace \times \aspace \rightarrow [0,\, \infty)$ represents the probability density of the next state $\stp\in\sspace$ given the current state $\st\in\sspace$ and action $\at\in\aspace$. The environment emits a reward $\reward: \sspace \times \aspace \rightarrow  [\rmin,\rmax]$ on each transition, which we will abbreviate as $\rt\triangleq \reward(\st, \at)$ to simplify notation. We will also use $\rho_\policy(\st)$ and $\rho_\policy(\st,\at)$ to denote the state and state-action marginals of the trajectory distribution induced by a policy $\policy(\at|\st)$. 

Our goal is to learn a policy $\policy(\at|\st)$. We can define the standard reinforcement learning objective in terms of the above quantities as
\vspace{-1mm}
\begin{equation}
\policy\opt_\text{std} = \arg\max_{\policy} \sum_t \E{(\st, \at)\sim \rho_\policy}{\reward(\st,\at)}.
\end{equation}
\vspace{-3mm}\\
Maximum entropy RL augments the reward with an entropy term, such that the optimal policy aims to maximize its entropy at each visited state:
\vspace{-5mm}
\begin{align}
\label{eq:maxent_objective}\\
\resizebox{1.0\hsize}{!}{$\policy\opt_\text{MaxEnt} \!=\! \arg\max_{\policy} \sum_t \E{(\st, \at) \sim \rho_\policy}{\reward(\st,\at) \!+\! \alpha\ent(\policy(\sdots|\st))},$}\notag
\end{align}
\noindent where $\alpha$ is an optional but convenient parameter that can be used to determine the relative importance of entropy and reward.\footnote{In principle, $1/\alpha$ can be folded into the reward function, eliminating the need for an explicit multiplier, but in practice, it is often convenient to keep $\alpha$ as a hyperparameter.} Optimization problems of this type have been explored in a number of prior works~\cite{kappen2005path,todorov2006linearly,ziebart2008maximum}, which are covered in more detail in \autoref{sec:related}. Note that this objective differs qualitatively from the behavior of Boltzmann exploration~\cite{sallans2004reinforcement} and PGQ~\cite{o2016pgq}, which greedily maximize entropy at the current time step, but do not explicitly optimize for policies that aim to reach \emph{states} where they will have high entropy in the future. This distinction is crucial, since the maximum entropy objective can be shown to maximize the entropy of the entire trajectory distribution for the policy $\policy$, while the greedy Boltzmann exploration approach does not~\cite{ziebart2008maximum,levine2014learning}. As we will discuss in \autoref{sec:experiments}, this maximum entropy formulation has a number of benefits, such as improved exploration in multimodal problems and better pretraining for later adaptation.

If we wish to extend either the conventional or the maximum entropy RL objective to infinite horizon problems, it is convenient to also introduce a discount factor $\discount$ to ensure that the sum of expected rewards (and entropies) is finite. In the context of policy search algorithms, the use of a discount factor is actually a somewhat nuanced choice, and writing down the precise objective that is optimized when using the discount factor is non-trivial \cite{thomas2014bias}. We defer the full derivation of the discounted objective to \aref{app:proofs}, since it is unwieldy to write out explicitly, but we will use the discount $\discount$ in the following derivations and in our final algorithm.

\vspace{-1mm}
\subsection{Soft Value Functions and Energy-Based Models}
\vspace{-1mm}
Optimizing the maximum entropy objective in \autoref{eq:maxent_objective} 
provides us with a framework for training stochastic policies, but we must still choose a representation for these policies. The choices in prior work include discrete multinomial distributions~\cite{o2016pgq} and Gaussian distributions~\cite{rawlik2012stochastic}. However, if we want to use a very general class of distributions that can represent complex, multimodal behaviors, we can instead opt for using general energy-based policies of the form
\vspace{-0mm}
\begin{equation}
\policy(\at|\st) \propto \exp\left(-\mathcal{E}(\st,\at)\right),
\end{equation}
\vspace{-3mm}\\
where $\mathcal{E}$ is an energy function that could be represented, for example, by a deep neural network. If we use a universal function approximator for $\mathcal{E}$, we can represent any distribution $\policy(\at|\st)$. There is a close connection between such energy-based models and \emph{soft} versions of value functions and Q-functions, where we set $\mathcal{E}(\st,\at) = -\frac{1}{\temp}\Qsoft(\st,\at)$ and use the following theorem:
\begin{theorem}
\label{the:ebm}
Let the soft Q-function be defined by 
\begin{align}
&\Qsoft\opt(\st, \at) = \reward_t + \label{eq:soft_state_action_value}\\
&\ \E{(\stp,\dots)\sim\rho_\policy\!}{\sum_{l=1}^\infty\discount^l\left(\reward_{t+l}\! +\! \temp\entropy\left(\policy\opt_\mathrm{MaxEnt}(\sdots|\state_{t+l})\right)\right)}\!,\notag
\end{align}
and soft value function by 
\vspace{-2mm}
\begin{align}
&\Vsoft\opt(\st) = \temp\log\int_{\aspace} \exp\left(\frac{1}{\temp}\Qsoft\opt(\st, \action')\right)d\action'.
\label{eq:soft_value}
\end{align}
Then the optimal policy for \autoref{eq:maxent_objective} is given by
\begin{equation}
\resizebox{0.88\hsize}{!}{$\policy\opt_\mathrm{MaxEnt}(\at|\st)\! =\! \exp\!\left(\frac{1}{\temp}\!\left(\Qsoft\opt(\st, \at)\! -\! \Vsoft\opt(\st)\right)\right).$}
\label{eq:softpolicy}
\end{equation}
\end{theorem}
\begin{proof}
\vspace{-2mm}
See \aref{app:max_ent_policy} as well as \cite{ziebart2010modeling}.
\end{proof}
\vspace{-2mm}
\autoref{the:ebm} connects the maximum entropy objective in \autoref{eq:maxent_objective} and energy-based models, where $\frac{1}{\temp}\Qsoft(\st, \at)$ acts as the negative energy, and $\frac{1}{\temp}\Vsoft(\st)$ serves as the log-partition function.
As with the standard Q-function and value function, we can relate the Q-function to the value function at a future state via a soft Bellman equation:
\begin{theorem}
\label{the:soft_bellman}
The soft Q-function in \autoref{eq:soft_state_action_value} satisfies the soft Bellman equation 
\begin{align}
\Qsoft\opt(\st, \at) = \rt + \discount\E{\stp\sim\pdyn}{\Vsoft\opt(\stp)},
\label{eq:soft_bellman}
\end{align}
where the soft value function $\Vsoft\opt$ is given by \autoref{eq:soft_value}.
\end{theorem}
\begin{proof}
\vspace{-3mm}
See \aref{app:soft_bellman_equation}, as well as  \cite{ziebart2010modeling}.
\end{proof}
\vspace{-3mm}
The soft Bellman equation is a generalization of the conventional (hard) equation, where we can recover the more standard equation as $\temp \rightarrow 0$, which causes \autoref{eq:soft_value} to approach a hard maximum over the actions. In the next section, we will discuss how we can use these identities to derive a Q-learning style algorithm for learning maximum entropy policies, and how we can make this practical for arbitrary Q-function representations via an approximate inference procedure.

\vspace{-2mm}
\section{Training Expressive Energy-Based Models via Soft Q-Learning}
\label{sec:training}
\vspace{-1mm}
In this section, we will present our proposed reinforcement learning algorithm, which is based on the soft Q-function described in the previous section, but can be implemented via a tractable stochastic gradient descent procedure with approximate sampling. We will first describe the general case of soft Q-learning, and then present the inference procedure that makes it tractable to use with deep neural network representations in high-dimensional continuous state and action spaces. In the process, we will relate this Q-learning procedure to inference in energy-based models and actor-critic algorithms.

\vspace{-3mm}
\subsection{Soft Q-Iteration}
\vspace{-2mm}

We can obtain a solution to \autoref{eq:soft_bellman} by iteratively updating estimates of $\Vsoft\opt$ and $\Qsoft\opt$. This leads to a fixed-point iteration that resembles Q-iteration:
\begin{theorem}
\label{the:soft_q_iteration}
Soft Q-iteration. Let $\Qsoft(\sdots, \sdots)$ and $\Vsoft(\sdots)$ be bounded and assume that $\int_{\aspace}\exp\!(\frac{1}{\temp}\Qsoft(\cdot, \action')\!)d\action'\!\! <\!\! \infty$ and that $\Qsoft\opt < \infty$ exists. Then the fixed-point iteration 
\begin{align}
\!\!\Qsoft(\st, \at) &\!\leftarrow\! \rt \!+\! \gamma \E{\stp\sim\pdyn}{\Vsoft(\stp)},\ \forall \st, \at\label{eq:q_fixed_point_iteration}\\
\Vsoft(\st) &\!\leftarrow\! \temp\log\!\int_{\aspace}\!\! \exp\left(\!\frac{1}{\temp}\Qsoft(\st, \action')\right)\!d\action'\!,\ \!\forall \st\label{eq:v_fixed_point_iteration}
\end{align}
converges to $\Qsoft\opt$ and $\Vsoft\opt$, respectively.
\end{theorem}
\begin{proof}
\vspace{-2mm}
See \aref{app:soft_bellman_equation} as well as \cite{fox2015taming}.
\end{proof}
\vspace{-2mm}
We refer to the updates in \autoref{eq:q_fixed_point_iteration} and \eqref{eq:v_fixed_point_iteration} as the soft Bellman backup operator that acts on the soft value function, and denote it by $\mathcal{T}$. The maximum entropy policy in \autoref{eq:softpolicy} can then be recovered by iteratively applying this operator until convergence. However, there are several practicalities that need to be considered in order to make use of the algorithm. First, the soft Bellman backup cannot be performed exactly in continuous or large state and action spaces, and second, sampling from the energy-based model in \autoref{eq:softpolicy} is intractable in general. We will address these challenges in the following sections.

\vspace{-2mm}
\subsection{Soft Q-Learning}
\vspace{-1mm}
This section discusses how the Bellman backup in \autoref{the:soft_q_iteration} can be implemented in a practical algorithm that uses a finite set of samples from the environment, resulting in a method similar to Q-learning. Since the soft Bellman backup is a contraction (see \aref{app:soft_bellman_equation}), the optimal value function is the fixed point of the Bellman backup, and we can find it by optimizing for a Q-function for which the soft Bellman error $|\mathcal{T}Q - Q|$ is minimized at all states and actions. While this procedure is still intractable due to the integral in \autoref{eq:v_fixed_point_iteration} and the infinite set of all states and actions, we can express it as a stochastic optimization, which leads to a stochastic gradient descent update procedure. We will model the soft Q-function with a function approximator with parameters $\qparams$ and denote it as $\Qsoftparams(\st,\at)$. 

To convert \autoref{the:soft_q_iteration} into a stochastic optimization problem, we first express the soft value function in terms of an expectation via importance sampling:
\vspace{-4mm}\\
\begin{align}
\Vsoft^\qparams(\st) = \temp \log \E{q_{\action'}}{\frac{\exp\left(\frac{1}{\temp}\Qsoft^\qparams(\st,\action')\right)}{q_{\action'}(\action')}},
\label{eq:soft_value_as_expectation}
\end{align}
\vspace{-4mm}\\
where $q_{\action'}$ can be an arbitrary distribution over the action space. Second, by noting the identity 
$g_1(x) = g_2(x)\ \forall x \in \mathbb{X}\ \Leftrightarrow\ \E{x\sim q}{(g_1(x) - g_2(x))^2} = 0$, where $q$ can be any strictly positive density function on $\mathbb{X}$, we can express the soft Q-iteration in an equivalent form as minimizing
\vspace{-4mm}\\
 \begin{equation}
\resizebox{0.88\hsize}{!}{$\qloss(\qparams)\! =\! \E{\st\sim q_\st,\at\sim q_\at\!}{\frac{1}{2}\left(\Qhatsoft^\qtargetparams(\st, \at)\!-\! \Qsoft^\qparams(\st, \at)\right)^2}\!,$}
\label{eq:qloss}
\end{equation}
\vspace{-4mm}\\
where $q_\st, q_\at$ are positive over $\sspace$ and $\aspace$ respectively, $\Qhatsoftparams(\st, \at) = \rt + \discount\EEE_{\stp\sim\pdyn}[\Vsoft^{\bar{\theta}}(\stp)]$ is a \emph{target} Q-value, with $\Vsoft^{\bar{\theta}}(\stp)$ given by \autoref{eq:soft_value_as_expectation} and $\theta$ being replaced by the target parameters, $\bar{\theta}$.

This stochastic optimization problem can be solved approximately using stochastic gradient descent using sampled states and actions. While the sampling distributions $q_\st$ and $q_\at$ can be arbitrary, we typically use real samples from rollouts of the current policy $\policy(\at|\st) \propto \exp\left(\frac{1}{\temp}\Qsoftparams(\st, \at)\right)$. For $q_{\action'}$ we have more options. A convenient choice is a uniform distribution. However, this choice can scale poorly to high dimensions. A better choice is to use the current policy, which produces an unbiased estimate of the soft value as can be confirmed by substitution. This overall procedure yields an iterative approach that optimizes over the Q-values, which we summarize in \autoref{sec:alg}.

However, in continuous spaces, we still need a tractable way to sample from the policy $\policy(\at|\st) \propto \exp\left(\frac{1}{\temp}\Qsoftparams(\st, \at)\right)$, both to take on-policy actions and, if so desired, to generate action samples for estimating the soft value function. Since the form of the policy is so general, sampling from it is intractable. We will therefore use an approximate sampling procedure, as discussed in the following section.

\vspace{-2mm}
\subsection{Approximate Sampling and Stein Variational Gradient Descent (SVGD)}
\label{sec:inference}
\vspace{-1mm}
In this section we describe how we can approximately sample from the soft Q-function. Existing approaches that sample from energy-based distributions generally fall into two categories: methods that use Markov chain Monte Carlo (MCMC) based sampling \cite{sallans2004reinforcement}, and methods that learn a stochastic sampling network trained to output approximate samples from the target distribution \cite{zhao2016energy,kim2016deep}. Since sampling via MCMC is not tractable when the inference must be performed online (\eg~when executing a policy), we will use a sampling network based on Stein variational gradient descent (SVGD) \cite{liu2016stein} and amortized SVGD \cite{wang2016learning}. Amortized SVGD has several intriguing properties: First, it provides us with a stochastic sampling network that we can query for extremely fast sample generation. Second, it can be shown to converge to an accurate estimate of the posterior distribution of an EBM. Third, the resulting algorithm, as we will show later, strongly resembles actor-critic algorithm, which provides for a simple and computationally efficient implementation and sheds light on the connection between our algorithm and prior actor-critic methods.

Formally, we want to learn a state-conditioned stochastic neural network $\at = f^\pparams(\smpl; \st)$, parametrized by $\pparams$, that maps noise samples $\smpl$ drawn from a normal Gaussian, or other arbitrary distribution, into unbiased action samples from the target EBM corresponding to $\Qsoftparams$.
We denote the induced distribution of the actions as $\policy^\pparams(\at | \st)$, and we want to find parameters $\pparams$ so that the induced distribution approximates the energy-based distribution in terms of the KL divergence
\vspace{-2mm}
\begin{align}
&\ploss(\pparams; \st) = \label{eq:ploss}\\
&\ \ \kl{\policy^\pparams({\sdots|\st})}{\exp\left(\frac{1}{\alpha}\left(\Qsoftparams(\st, \sdots) - \Vsoftparams\right)\right)}\notag.
\end{align}
Suppose we ``perturb'' a set of independent samples $\ati = f^\pparams(\smpli; \st)$ in appropriate directions $\Delta f^\pparams(\smpli; \st)$, the induced KL divergence can be reduced. Stein variational gradient descent \cite{liu2016stein} provides the most greedy directions as a functional
\begin{align}
&\resizebox{1.0\hsize}{!}{
$\Delta f^\pparams(\sdots; \st) =\label{eq:stein_gradient}
\mathbb{E}_{\at\sim\policy^\pparams}\left[\kernel(\at, f^\pparams(\sdots; \st))\nabla_{\action'} \Qsoftparams(\st, \action')\big|_{\action' = \at} \right.$} \\[-5mm] 
&\hspace{18mm} + \alpha\left.\nabla_{\action'}\kernel(\action', f^\pparams(\sdots; \st))\big|_{\action' = \at}\right],\notag
\end{align}
where $\kernel$ is a kernel function (typically Gaussian, see details in \aref{app: hyperparameters}). 
To be precise, $\Delta f^\pparams$ is the optimal direction in the reproducing kernel Hilbert space of $\kappa$, and is thus not strictly speaking the gradient of \autoref{eq:ploss}, but it turns out that we can set $\frac{\partial\ploss}{\partial\at} \propto \Delta f^\pparams$ as explained in \cite{wang2016learning}.
With this assumption, we can use the chain rule and backpropagate the Stein variational gradient into the policy network according to
\begin{align}
\frac{\partial \ploss(\pparams; \st)}{\partial \pparams} \propto \E{\xi}{\Delta f^\pparams(\xi;\st)\frac{\partial f^\pparams(\xi; \st)}{\partial \pparams}},
\label{eq:actor_gradient}
\end{align}
and use any gradient-based optimization method to learn the optimal sampling network parameters. 
The sampling network $f^\pparams$ can be viewed as an actor in an actor-critic algorithm. We will discuss this connection in \autoref{sec:related}, but first we will summarize our complete maximum entropy policy learning algorithm.

\vspace{-2mm}
\subsection{Algorithm Summary}
\label{sec:alg}
\vspace{-1mm}
To summarize, we propose the soft Q-learning algorithm for learning maximum entropy policies in continuous domains. The algorithm proceeds by alternating between collecting new experience from the environment, and updating the soft Q-function and sampling network parameters.  The experience is stored in a replay memory buffer $\mathcal{D}$ as standard in deep Q-learning \cite{mnih2013playing}, and the parameters are updated using random minibatches from this memory. The soft Q-function updates use a delayed version of the target values \cite{mnih2015human}. For optimization, we use the ADAM \cite{kingma2014adam} optimizer and empirical estimates of the gradients, which we denote by $\hat \nabla$. The exact formulae used to compute the gradient estimates is deferred to \aref{app:implementation}, which also discusses other implementation details, but we summarize an overview of soft Q-learning in \autoref{alg:soft-q-learning}.

\begin{algorithm}[htb]
\caption{Soft Q-learning}
\label{alg:soft-q-learning}
\begin{algorithmic}
\STATE
\ \ $\qparams, \pparams  \sim$ some initialization distributions.\\
\ \ Assign target parameters: $\bar\qparams \leftarrow \qparams$, $\bar\pparams \leftarrow \pparams$. \\
\ \ $\mathcal{D} \leftarrow $ empty replay memory.
\vspace{1mm}\\
\FOR{each epoch}
   \FOR{each $t$}
        \STATE \textbf{Collect experience}\\
   		\ \ Sample an action for $\st$ using $f^\pparams$: \\
   		\ \ \ \ $\at \leftarrow f^\pparams(\smpl; \st)$ where $\smpl \sim \gauss\left(\zeros, \eye\right)$.\\
        \ \ Sample next state from the environment:\\
        \ \ \ \ $\stp \sim \pdyn(\stp|\st, \at)$.\\
        \ \ Save the new experience in the replay memory:\\
        \ \ \ \ $\mathcal{D} \leftarrow \mathcal{D} \cup \left\{(\st, \at, \reward(\st, \at), \stp)\right\}.$

        \STATE\textbf{Sample a minibatch from the replay memory}\\
        \ \ $\{(\sti, \ati,\rti, \stpi)\}_{i=0}^{N}\sim \mathcal{D}.$

        \STATE\textbf{Update the soft Q-function parameters}\\
        \ \ Sample $\{\action^{(i,j)}\}_{j=0}^M \sim q_{\action'}$ for each $\stpi$.\\
        \ \ Compute empirical soft values $\hat{V}_{\mathrm{soft}}^{\bar{\theta}}(\stpi)$ in \eqref{eq:soft_value_as_expectation}.\\
        \ \ Compute  empirical gradient $\hat\nabla_\qparams\qloss$ of \eqref{eq:qloss}.\\
        \ \ Update $\qparams$ according to $\hat\nabla_\qparams\qloss$ using ADAM.

        \STATE \textbf{Update policy}\\
        \ \ Sample $\{\smplij\}_{j=0}^M \sim \gauss\left(\zeros, \eye\right)$ for each $\sti$.\\
        \ \ Compute actions $\atij = f^\pparams(\smplij, \sti)$.\\
        \ \ Compute $\Delta f^\pparams$ using empirical estimate of \eqref{eq:stein_gradient}.\\
        \ \ Compute empiricial estimate of \eqref{eq:actor_gradient}: $\hat\nabla_\pparams\ploss$.\\
        \ \ Update $\pparams$ according to $\hat\nabla_\pparams\ploss$ using ADAM.
    \ENDFOR
    \IF{epoch \textit{mod} update\_interval $= 0$}
    \STATE Update target parameters: $\bar\qparams \leftarrow \qparams$, $\bar\pparams \leftarrow \pparams$.
    \ENDIF
\ENDFOR
\end{algorithmic}
\end{algorithm}

\vspace{-2mm}
\section{Related Work}
\label{sec:related}
\vspace{-1mm}

Maximum entropy policies emerge as the solution when we cast optimal control as probabilistic inference. In the case of linear-quadratic systems, the mean of the maximum entropy policy is exactly the optimal deterministic policy~\cite{todorov2008general}, which has been exploited to construct practical path planning methods based on iterative linearization and probabilistic inference techniques~\cite{toussaint2009robot}.
In discrete state spaces, the maximum entropy policy can be obtained exactly. This has been explored in the context of linearly solvable MDPs \cite{todorov2006linearly}
and, in the case of inverse reinforcement learning, MaxEnt IRL \cite{ziebart2008maximum}.
In continuous systems and continuous time, path integral control studies maximum entropy policies and maximum entropy planning \cite{kappen2005path}. In contrast to these prior methods, our work is focused on extending the maximum entropy policy search framework to high-dimensional continuous spaces and highly multimodal objectives, via expressive general-purpose energy functions represented by deep neural networks. A number of related methods have also used maximum entropy policy optimization as an intermediate step for optimizing policies under a standard expected reward objective \cite{peters2010relative,neumann2011variational,rawlik2012stochastic,fox2015taming}. Among these, the work of \citet{rawlik2012stochastic} resembles ours in that it also makes use of a temporal difference style update to a soft Q-function. However, unlike this prior work, we focus on general-purpose energy functions with approximate sampling, rather than analytically normalizable distributions. A recent work \cite{liu2017stein} also considers an entropy regularized objective, though the entropy is on policy parameters, not on sampled actions. Thus the resulting policy may not represent an arbitrarily complex multi-modal distribution with a single parameter. The form of our sampler resembles the stochastic networks proposed in recent work on hierarchical learning~\cite{florensa2017stochastic}. However this prior work uses a task-specific reward bonus system to encourage stochastic behavior, while our approach is derived from optimizing a general maximum entropy objective.
A closely related concept to maximum entropy policies is Boltzmann exploration, which uses the exponential of the standard Q-function as the probability of an action~\cite{kaelbling1996reinforcement}.
A number of prior works have also explored representing policies as energy-based models, with the Q-value obtained from an energy model such as a restricted Boltzmann machine (RBM) \cite{sallans2004reinforcement,elfwing2010free,otsuka2010free,heess2012actor}. Although these methods are closely related, they have not, to our knowledge, been extended to the case of deep network models, have not made extensive use of approximate inference techniques, and have not been demonstrated on the complex continuous tasks. More recently, \citet{o2016pgq} drew a connection between Boltzmann exploration and entropy-regularized policy gradient, though in a theoretical framework that differs from maximum entropy policy search: unlike the full maximum entropy framework, the approach of \citet{o2016pgq} only optimizes for maximizing entropy at the current time step, rather than planning for visiting future states where entropy will be further maximized. This prior method also does not demonstrate learning complex multi-modal policies in continuous action spaces.

Although we motivate our method as Q-learning, its structure resembles an actor-critic algorithm. It is particularly instructive to observe the connection between our approach and the deep deterministic policy gradient method (DDPG)~\cite{lillicrap2015continuous}, which updates a Q-function critic according to (hard) Bellman updates, and then backpropagates the Q-value gradient into the actor, similarly to NFQCA \cite{hafner2011reinforcement}. Our actor update differs only in the addition of the $\kappa$ term. Indeed, without this term, our actor would estimate a maximum a posteriori (MAP) action, rather than capturing the entire EBM distribution. This suggests an intriguing connection between our method and DDPG: if we simply modify the DDPG critic updates to estimate soft Q-values, we recover the MAP variant of our method. Furthermore, this connection allows us to cast DDPG as simply an approximate Q-learning method, where the actor serves the role of an approximate maximizer. This helps explain the good performance of DDPG on off-policy data. We can also make a connection between our method and policy gradients. In \aref{app:pg}, we show that the policy gradient for a policy represented as an energy-based model closely corresponds to the update in soft Q-learning. Similar derivation is presented in a concurrent work \cite{schulman2017equivalence}.

\vspace{-3mm}
\section{Experiments}
\label{sec:experiments}
\vspace{-2mm}

Our experiments aim to answer the following questions: (1) Does our soft Q-learning method accurately capture a multi-modal policy distribution? (2) Can soft Q-learning with energy-based policies aid exploration for complex tasks that require tracking multiple modes? (3) Can a maximum entropy policy serve as a good initialization for finetuning on different tasks, when compared to pretraining with a standard deterministic objective? We compare our algorithm to DDPG~\cite{lillicrap2015continuous}, which has been shown to achieve better sample efficiency on the continuous control problems that we consider than other recent techniques such as REINFORCE \cite{williams1992simple}, TRPO \cite{schulman2015trust}, and A3C \cite{mnih2016asynchronous}. This comparison is particularly interesting since, as discussed in \autoref{sec:related}, DDPG closely corresponds to a deterministic maximum a posteriori variant of our method.
The detailed experimental setup can be found in \aref{app:experiments}. Videos of all experiments\footnote{\href{https://sites.google.com/view/softqlearning/home}{https://sites.google.com/view/softqlearning/home}} and example source code\footnote{\href{https://github.com/haarnoja/softqlearning}{https://github.com/haarnoja/softqlearning}} are available online.

\vspace{-2mm}
\subsection{Didactic Example: Multi-Goal Environment}
\vspace{-1mm}

In order to verify that amortized SVGD can correctly draw samples from energy-based policies of the form $\exp\left(\Qsoftparams(s,a)\right)$, and that our complete algorithm can successful learn to represent multi-modal behavior, we designed a simple ``multi-goal'' environment, in which the agent is a 2D point mass trying to reach one of four symmetrically placed goals. The reward is defined as a mixture of Gaussians, with means placed at the goal positions. An optimal strategy is to go to an arbitrary goal, and the optimal maximum entropy policy should be able to choose each of the four goals at random. The final policy obtained with our method is illustrated in \autoref{fig:multigoal}. The Q-values indeed have complex shapes, being unimodal at $s = (-2,0)$, convex at $s = (0,0)$, and bimodal at $s = (2.5, 2.5)$. The stochastic policy samples actions closely following the energy landscape, hence learning diverse trajectories that lead to all four goals. In comparison, a policy trained with DDPG randomly commits to a single goal.
\begin{figure}[bt]
    \centering
    \includegraphics[width=0.48\textwidth]{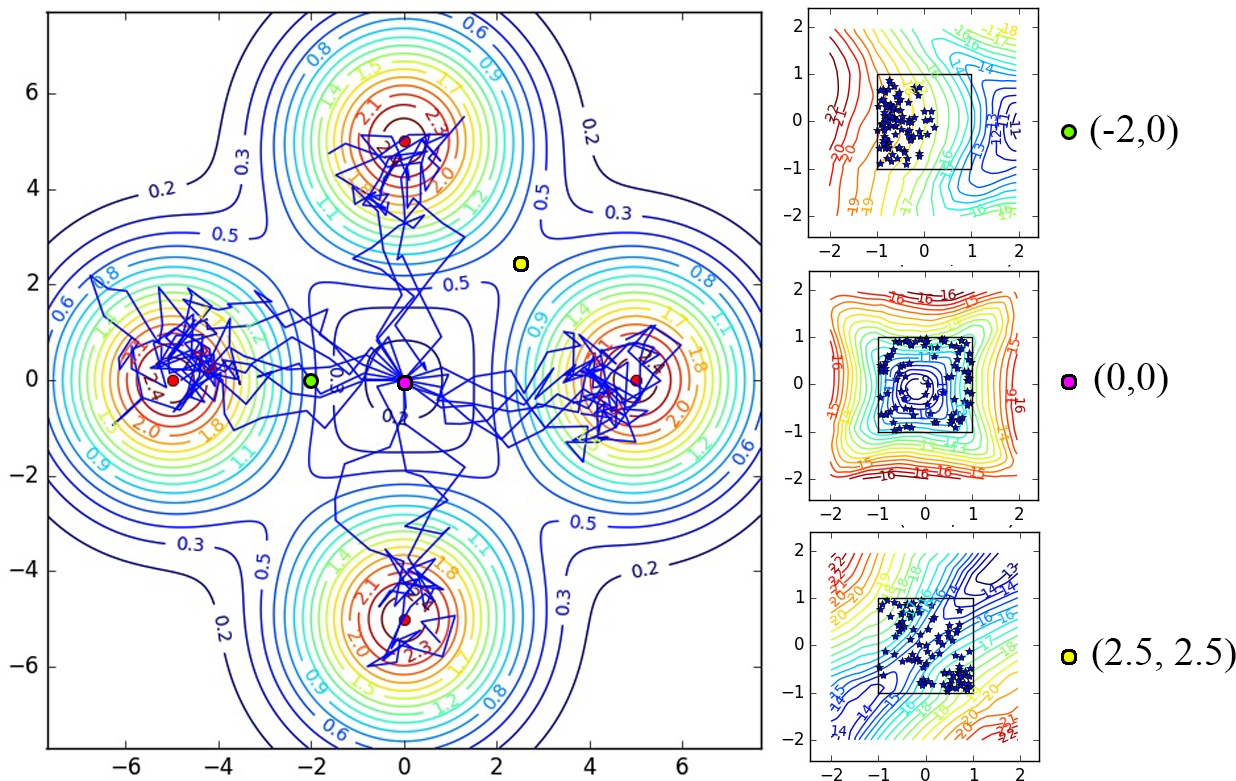}
    \caption{
    Illustration of the 2D multi-goal environment. Left: trajectories from a policy learned with our method (solid blue lines). The $x$ and $y$ axes correspond to 2D positions (states). The agent is initialized at the origin. The goals are depicted as red dots, and the level curves show the reward.
    Right: Q-values at three selected states, depicted by level curves (red: high values, blue: low values). The $x$ and $y$ axes correspond to 2D velocity (actions) bounded between -1 and 1. Actions sampled from the policy are shown as blue stars. Note that, in regions (e.g. $(2.5, 2.5)$) between the goals, the method chooses multimodal actions.
    \label{fig:multigoal}
    }
\end{figure}

\vspace{-2mm}
\subsection{Learning Multi-Modal Policies for Exploration}
\vspace{-1mm}
Though not all environments have a clear multi-modal reward landscape as in the ``multi-goal'' example, multi-modality is prevalent in a variety of tasks. For example, a chess player might try various strategies before settling on one that seems most effective, and an agent navigating a maze may need to try various paths before finding the exit. During the learning process, it is often best to keep trying multiple available options until the agent is confident that one of them is the best (similar to a bandit problem \cite{lai1985asymptotically}). 
However, deep RL algorithms for continuous control typically use unimodal action distributions, which are not well suited to capture such multi-modality. As a consequence, such algorithms may prematurely commit to one mode and converge to suboptimal behavior.
To evaluate how maximum entropy policies might aid exploration, we constructed simulated continuous control environments where tracking multiple modes is important for success.
The first experiment uses a simulated swimming snake (see \autoref{fig:swimmer_and_ant}), which receives a reward equal to its speed along the $x$-axis, either forward or backward. However, once the swimmer swims far enough forward, it crosses a ``finish line'' and receives a larger reward. Therefore, the best learning strategy is to explore in both directions until the bonus reward is discovered, and then commit to swimming forward. As illustrated in \autoref{fig:swimmer_all_dist_plots} in \aref{app:additional-results}, our method is able to recover this strategy, keeping track of both modes until the finish line is discovered. All stochastic policies eventually commit to swimming forward. The deterministic DDPG method shown in the comparison commits to a mode prematurely, with only 80\% of the policies converging on a forward motion, and 20\% choosing the suboptimal backward mode.

\begin{figure}
    \centering
    \subfigure[Swimming snake]
    {\includegraphics[width=0.45\columnwidth]{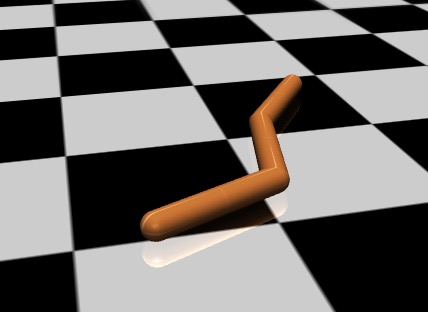}}
    \subfigure[Quadrupedal robot]
    {\includegraphics[width=0.45\columnwidth]{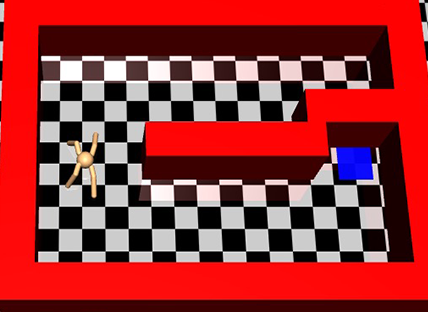}}
    \caption{Simulated robots used in our experiments. \label{fig:swimmer_and_ant}}
\end{figure}

The second experiment studies a more complex task with a continuous range of equally good options prior to discovery of a sparse reward goal. In this task, a quadrupedal 3D robot (adapted from \citet{schulman2015high})
needs to find a path through a maze to a target position (see \autoref{fig:swimmer_and_ant}). The reward function is a Gaussian centered at the target.
The agent may choose either the upper or lower passage, which appear identical at first, but the upper passage is blocked by a barrier. Similar to the swimmer experiment, the optimal strategy requires exploring both directions and choosing the better one. \autoref{fig:ant_maze}(b) compares the performance of DDPG and our method. The curves show the minimum distance to the target achieved so far and the threshold equals the minimum possible distance if the robot chooses the upper passage. Therefore, successful exploration means reaching below the threshold. All policies trained with our method manage to succeed, while only $60\%$ policies trained with DDPG converge to choosing the lower passage. 

\begin{figure}
    \subfigure[Swimmer (higher is better)]
    {\includegraphics[height=0.33\columnwidth]{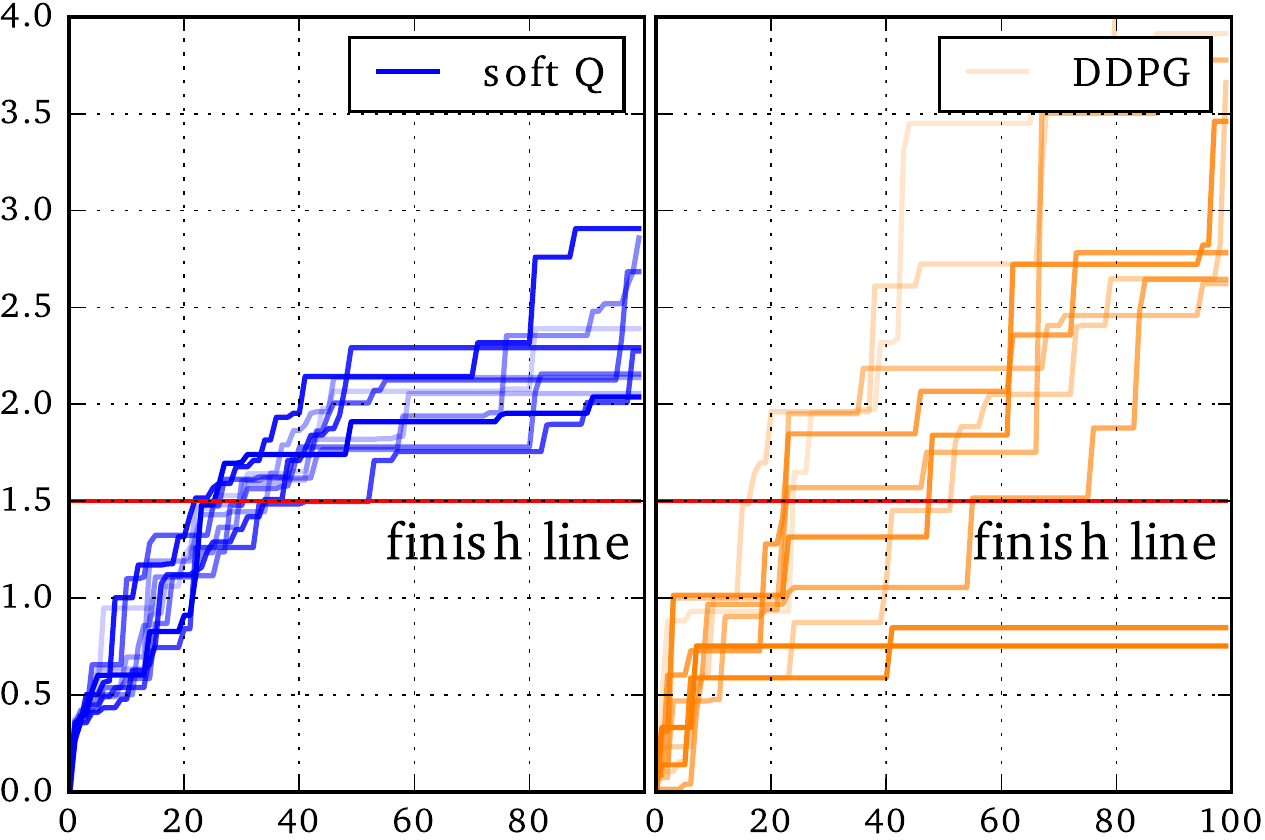}}
    \subfigure[Quadruped (lower is better)]
    {\includegraphics[height=0.33\columnwidth]{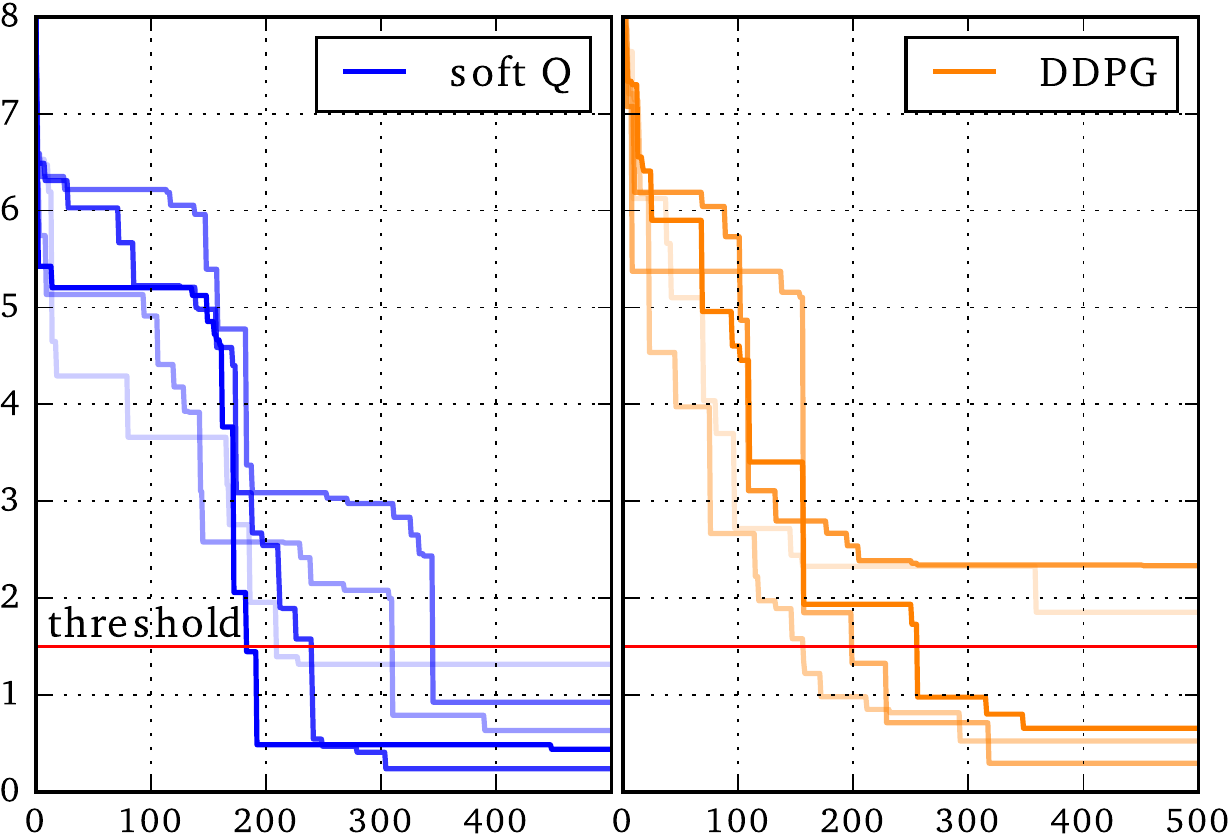}}
    \caption{
    Comparison of soft Q-learning and DDPG on the swimmer snake task and the quadrupedal robot maze task. (a) Shows the maximum traveled forward distance since the beginning of training for several runs of each algorithm; there is a large reward after crossing the finish line. (b) Shows our method was able to reach a low distance to the goal faster and more consistently. The different lines show the minimum distance to the goal since the beginning of training. For both domains, all runs of our method cross the threshold line, acquiring the more optimal strategy, while some runs of DDPG do not.
    \label{fig:ant_maze}
    }
\end{figure}

\subsection{Accelerating Training on Complex Tasks with Pretrained Maximum Entropy Policies}
A standard way to accelerate deep neural network training is task-specific initialization \cite{goodfellow2016deep_ch8.7.4}, where a network trained for one task is used as initialization for another task. The first task might be something highly general, such as classifying a large image dataset, while the second task might be more specific, such as fine-grained classification with a small dataset. Pretraining has also been explored in the context of RL \cite{shelhamer2016loss}. However, in RL, near-optimal policies are often near-deterministic, which makes them poor initializers for new tasks. In this section, we explore how our energy-based policies can be trained with fairly broad objectives to produce an initializer for more quickly learning more specific tasks. 
\begin{wrapfigure}{r}{.25\columnwidth}
\setlength{\unitlength}{0.5\columnwidth}
\vspace{-0.18in}\hspace{-0.22in}
\includegraphics[width=0.30\columnwidth]{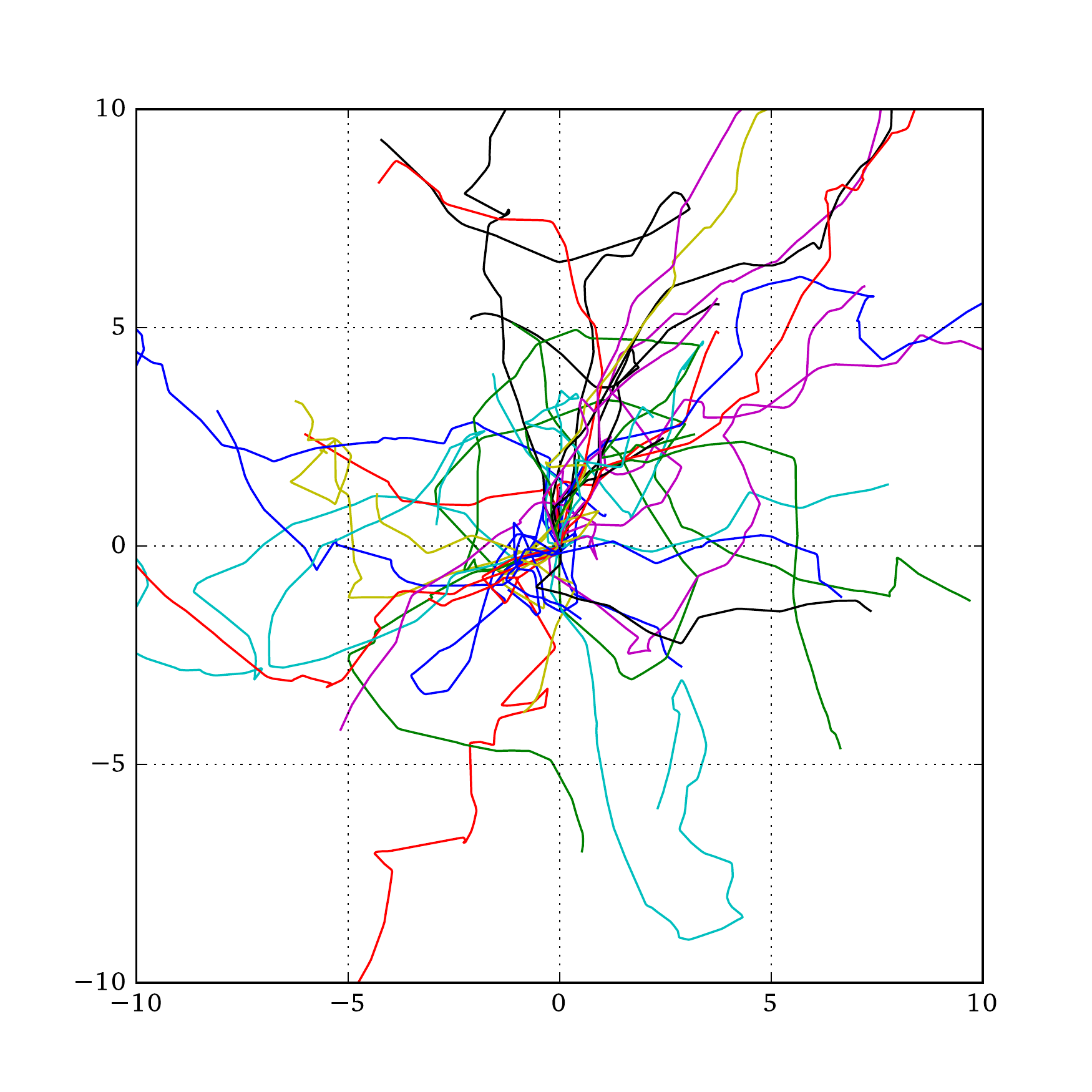}
\vspace{-0.20in}
\end{wrapfigure}
We demonstrate this on a variant of the quadrupedal robot task.
The pretraining phase involves learning to locomote in an arbitrary direction, with a reward that simply equals the speed of the center of mass. The resulting policy moves the agent quickly to an randomly chosen direction. An overhead plot of the center of mass traces is shown above to illustrate this. This pretraining is similar in some ways to recent work on modulated controllers \cite{heess2016learning} and hierarchical models \cite{florensa2017stochastic}. However, in contrast to these prior works, we do not require any task-specific high-level goal generator or reward.%

\autoref{fig:pretrain_envs} also shows a variety of test environments that we used to finetune the running policy for a specific task. In the  hallway environments, the agent receives the same reward, but the walls block sideways motion, so the optimal solution requires learning to run in a particular direction. Narrow hallways require choosing a more specific direction, but also allow the agent to use the walls to funnel itself. The U-shaped maze requires the agent to learn a curved trajectory in order to arrive at the target, with the reward given by a Gaussian bump at the target location.

\begin{figure}
    \subfigure[]
    {\includegraphics[height=0.22\columnwidth]{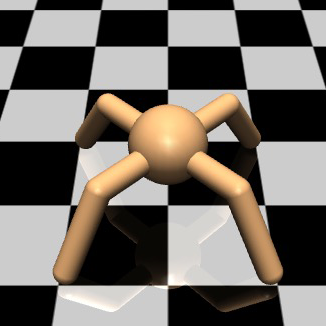}}
    \subfigure[]
    {\includegraphics[height=0.22\columnwidth]{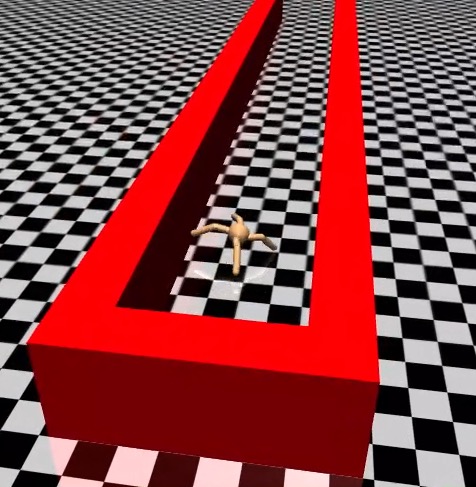}}
    \subfigure[]
    {\includegraphics[height=0.22\columnwidth]{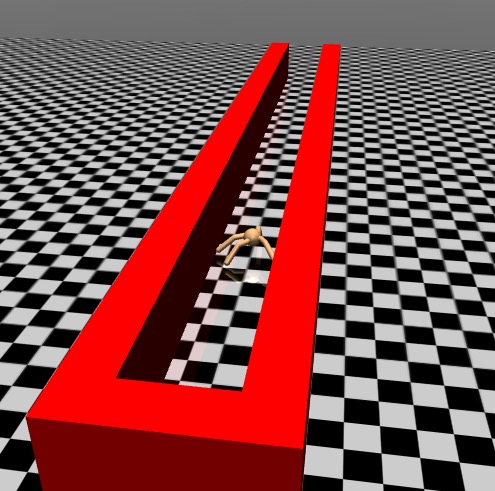}}
    \subfigure[]
    {\includegraphics[height=0.22\columnwidth]{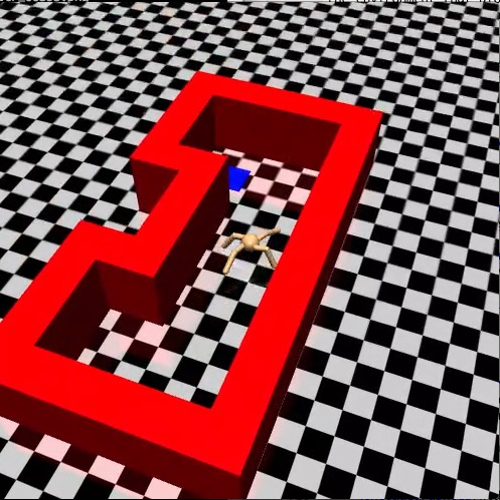}}
    \vspace*{-4mm}
    \caption{Quadrupedal robot (a) was trained to walk in random directions in an empty pretraining environment (details in \autoref{fig:pretrain_trajs}, see \aref{app:additional-results}), and then finetuned on a variety of tasks, including a wide (b), narrow (c), and U-shaped hallway (d).
    \label{fig:pretrain_envs}
    \vspace{-3mm}
    }
\end{figure}

As illustrated in \autoref{fig:pretrain_trajs} in \aref{app:additional-results}, the pretrained policy explores the space extensively and in all directions. This gives a good initialization for the policy, allowing it to learn the behaviors in the test environments more quickly than training a policy with DDPG from a random initialization, as shown in \autoref{fig:pretraining_curves}. We also evaluated an alternative pretraining method based on deterministic policies learned with DDPG. However, deterministic pretraining chooses an arbitrary but consistent direction in the training environment, providing a poor initialization for finetuning to a specific task, as shown in the results plots. 

\begin{figure}
    \includegraphics[width=\columnwidth]{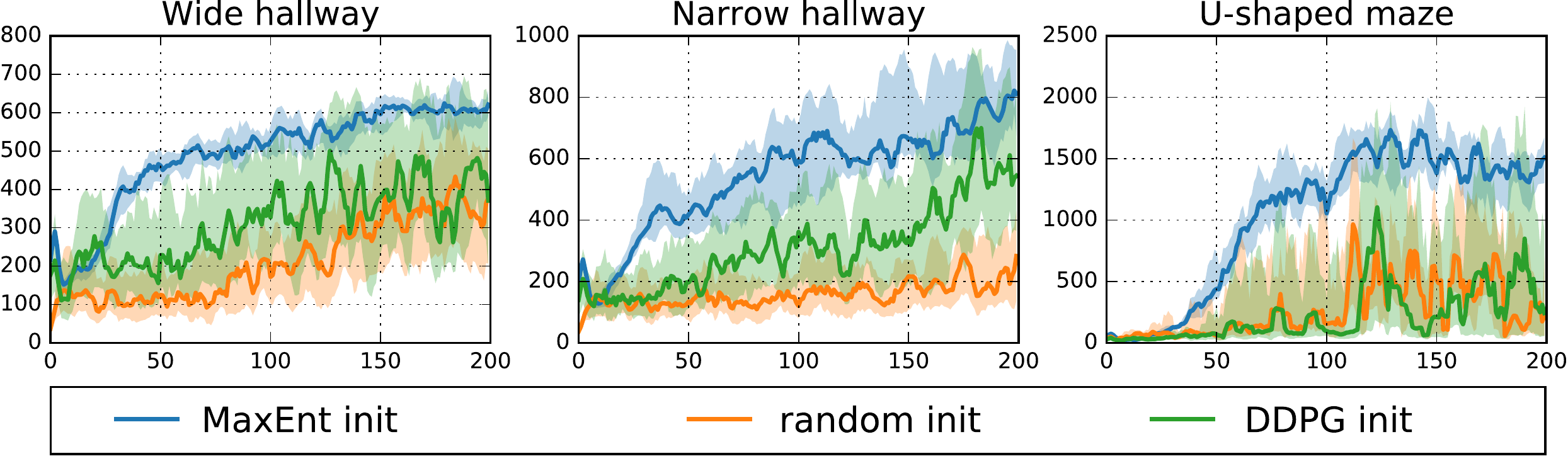}
    \vspace*{-5mm}
    \caption{Performance in the downstream task with fine-tuning (MaxEnt) or training from scratch (DDPG). The $x$-axis shows the training iterations. The y-axis shows the average discounted return. Solid lines are average values over 10 random seeds. Shaded regions correspond to one standard deviation.
    \label{fig:pretraining_curves}
    \vspace{-3mm}
    }
\end{figure}

\vspace{-3mm}
\section{Discussion and Future Work}
\vspace{-2mm}

We presented a method for learning stochastic energy-based policies with approximate inference via Stein variational gradient descent (SVGD). Our approach can be viewed as a type of soft Q-learning method, with the additional contribution of using approximate inference to obtain complex multimodal policies. 
The sampling network trained as part of SVGD can also be viewed as tking the role of an actor in an actor-critic algorithm. 
Our experimental results show that our method can effectively capture complex multi-modal behavior on problems ranging from toy point mass tasks to complex torque control of simulated walking and swimming robots. The applications of training such stochastic policies include improved exploration in the case of multimodal objectives and compositionality via pretraining general-purpose stochastic policies that can then be efficiently finetuned into task-specific behaviors.

While our work explores some potential applications of energy-based policies with approximate inference, an exciting avenue for future work would be to further study their capability to represent complex behavioral repertoires and their potential for composability. In the context of linearly solvable MDPs, several prior works have shown that policies trained for different tasks can be composed to create new optimal policies~\cite{da2009linear,todorov2009compositionality}. While these prior works have only explored simple, tractable representations, our method could be used to extend these results to complex and highly multi-modal deep neural network models, making them suitable for composable control of complex high-dimensional systems, such as humanoid robots. This composability could be used in future work to create a huge variety of near-optimal skills from a collection of energy-based policy building blocks.

\section{Acknowledgements}
We thank Qiang Liu for insightful discussion of SVGD, and thank Vitchyr Pong and Shane Gu for help with implementing DDPG. Haoran Tang and Tuomas Haarnoja are supported by Berkeley Deep Drive.

\bibliography{refs}

\begin{thebibliography}{50}
\providecommand{\natexlab}[1]{#1}
\providecommand{\url}[1]{\texttt{#1}}
\expandafter\ifx\csname urlstyle\endcsname\relax
  \providecommand{\doi}[1]{doi: #1}\else
  \providecommand{\doi}{doi: \begingroup \urlstyle{rm}\Url}\fi

\bibitem[Da~Silva et~al.(2009)Da~Silva, Durand, and Popovi{\'c}]{da2009linear}
Da~Silva, M., Durand, F., and Popovi{\'c}, J.
\newblock Linear {B}ellman combination for control of character animation.
\newblock \emph{ACM Trans. on Graphs}, 28\penalty0 (3):\penalty0 82, 2009.

\bibitem[Daniel et~al.(2012)Daniel, Neumann, and
  Peters]{daniel2012hierarchical}
Daniel, C., Neumann, G., and Peters, J.
\newblock Hierarchical relative entropy policy search.
\newblock In \emph{AISTATS}, pp.\  273--281, 2012.

\bibitem[Elfwing et~al.(2010)Elfwing, Otsuka, Uchibe, and
  Doya]{elfwing2010free}
Elfwing, S., Otsuka, M., Uchibe, E., and Doya, K.
\newblock Free-energy based reinforcement learning for vision-based navigation
  with high-dimensional sensory inputs.
\newblock In \emph{Int. Conf. on Neural Information Processing}, pp.\
  215--222. Springer, 2010.

\bibitem[Florensa et~al.(2017)Florensa, Duan, and P.]{florensa2017stochastic}
Florensa, C., Duan, Y., and P., Abbeel.
\newblock Stochastic neural networks for hierarchical reinforcement learning.
\newblock In \emph{Int. Conf. on Learning Representations}, 2017.

\bibitem[Fox et~al.(2016)Fox, Pakman, and Tishby]{fox2015taming}
Fox, R., Pakman, A., and Tishby, N.
\newblock Taming the noise in reinforcement learning via soft updates.
\newblock In \emph{Conf. on Uncertainty in Artificial Intelligence}, 2016.

\bibitem[Goodfellow et~al.(2016)Goodfellow, Bengio, and
  Courville]{goodfellow2016deep_ch8.7.4}
Goodfellow, Ian, Bengio, Yoshua, and Courville, Aaron.
\newblock Deep learning.
\newblock chapter 8.7.4. MIT Press, 2016.
\newblock \url{http://www.deeplearningbook.org}.

\bibitem[Gu et~al.(2016{\natexlab{a}})Gu, Lillicrap, Ghahramani, Turner, and
  Levine]{gu2016q}
Gu, S., Lillicrap, T., Ghahramani, Z., Turner, R.~E., and Levine, S.
\newblock Q-prop: Sample-efficient policy gradient with an off-policy critic.
\newblock \emph{arXiv preprint arXiv:1611.02247}, 2016{\natexlab{a}}.

\bibitem[Gu et~al.(2016{\natexlab{b}})Gu, Lillicrap, Sutskever, and
  Levine]{gu2016continuous}
Gu, S., Lillicrap, T., Sutskever, I., and Levine, S.
\newblock Continuous deep {Q}-learning with model-based acceleration.
\newblock In \emph{Int. Conf. on Machine Learning}, pp.\  2829--2838,
  2016{\natexlab{b}}.

\bibitem[Hafner \& Riedmiller(2011)Hafner and
  Riedmiller]{hafner2011reinforcement}
Hafner, R. and Riedmiller, M.
\newblock Reinforcement learning in feedback control.
\newblock \emph{Machine Learning}, 84\penalty0 (1-2):\penalty0 137--169, 2011.

\bibitem[Heess et~al.(2012)Heess, Silver, and Teh]{heess2012actor}
Heess, N., Silver, D., and Teh, Y.~W.
\newblock Actor-critic reinforcement learning with energy-based policies.
\newblock In \emph{Workshop on Reinforcement Learning}, pp.\ ~43. Citeseer,
  2012.

\bibitem[Heess et~al.(2016)Heess, Wayne, Tassa, Lillicrap, Riedmiller, and
  Silver]{heess2016learning}
Heess, N., Wayne, G., Tassa, Y., Lillicrap, T., Riedmiller, M., and Silver, D.
\newblock Learning and transfer of modulated locomotor controllers.
\newblock \emph{arXiv preprint arXiv:1610.05182}, 2016.

\bibitem[Jaderberg et~al.(2016)Jaderberg, Mnih, Czarnecki, Schaul, Leibo,
  Silver, and Kavukcuoglu]{jaderberg2016reinforcement}
Jaderberg, M., Mnih, V., Czarnecki, W.~M., Schaul, T., Leibo, J.~Z., Silver,
  D., and Kavukcuoglu, K.
\newblock Reinforcement learning with unsupervised auxiliary tasks.
\newblock \emph{arXiv preprint arXiv:1611.05397}, 2016.

\bibitem[Kaelbling et~al.(1996)Kaelbling, Littman, and
  Moore]{kaelbling1996reinforcement}
Kaelbling, L.~P., Littman, M.~L., and Moore, A.~W.
\newblock Reinforcement learning: {A} survey.
\newblock \emph{Journal of artificial intelligence research}, 4:\penalty0
  237--285, 1996.

\bibitem[Kakade(2002)]{kakade2002natural}
Kakade, S.
\newblock A natural policy gradient.
\newblock \emph{Advances in Neural Information Processing Systems}, 2:\penalty0
  1531--1538, 2002.

\bibitem[Kappen(2005)]{kappen2005path}
Kappen, H.~J.
\newblock Path integrals and symmetry breaking for optimal control theory.
\newblock \emph{Journal of Statistical Mechanics: Theory And Experiment},
  2005\penalty0 (11):\penalty0 P11011, 2005.

\bibitem[Kim \& Bengio(2016)Kim and Bengio]{kim2016deep}
Kim, T. and Bengio, Y.
\newblock Deep directed generative models with energy-based probability
  estimation.
\newblock \emph{arXiv preprint arXiv:1606.03439}, 2016.

\bibitem[Kingma \& Ba(2015)Kingma and Ba]{kingma2014adam}
Kingma, D. and Ba, J.
\newblock Adam: A method for stochastic optimization.
\newblock 2015.

\bibitem[Lai \& Robbins(1985)Lai and Robbins]{lai1985asymptotically}
Lai, T.~L. and Robbins, H.
\newblock Asymptotically efficient adaptive allocation rules.
\newblock \emph{Advances in Applied Mathematics}, 6\penalty0 (1):\penalty0
  4--22, 1985.

\bibitem[Levine \& Abbeel(2014)Levine and Abbeel]{levine2014learning}
Levine, S. and Abbeel, P.
\newblock Learning neural network policies with guided policy search under
  unknown dynamics.
\newblock In \emph{Advances in Neural Information Processing Systems}, pp.\
  1071--1079, 2014.

\bibitem[Levine et~al.(2016)Levine, Finn, Darrell, and Abbeel]{levine2016end}
Levine, S., Finn, C., Darrell, T., and Abbeel, P.
\newblock End-to-end training of deep visuomotor policies.
\newblock \emph{Journal of Machine Learning Research}, 17\penalty0
  (39):\penalty0 1--40, 2016.

\bibitem[Lillicrap et~al.(2015)Lillicrap, Hunt, Pritzel, Heess, Erez, Tassa,
  Silver, and Wierstra]{lillicrap2015continuous}
Lillicrap, T.~P., Hunt, J.~J., Pritzel, A., Heess, N., Erez, T., Tassa, Y.,
  Silver, D., and Wierstra, D.
\newblock Continuous control with deep reinforcement learning.
\newblock \emph{arXiv preprint arXiv:1509.02971}, 2015.

\bibitem[Liu \& Wang(2016)Liu and Wang]{liu2016stein}
Liu, Q. and Wang, D.
\newblock Stein variational gradient descent: A general purpose bayesian
  inference algorithm.
\newblock In \emph{Advances In Neural Information Processing Systems}, pp.\
  2370--2378, 2016.

\bibitem[Liu et~al.(2017)Liu, Ramachandran, Liu, and Peng]{liu2017stein}
Liu, Y., Ramachandran, P., Liu, Q., and Peng, J.
\newblock Stein variational policy gradient.
\newblock \emph{arXiv preprint arXiv:1704.02399}, 2017.

\bibitem[Mnih et~al.(2013)Mnih, Kavukcuoglu, Silver, Graves, Antonoglou,
  Wierstra, and Riedmiller]{mnih2013playing}
Mnih, V., Kavukcuoglu, K., Silver, D., Graves, A., Antonoglou, I., Wierstra,
  D., and Riedmiller, M.
\newblock Playing atari with deep reinforcement learning.
\newblock \emph{arXiv preprint arXiv:1312.5602}, 2013.

\bibitem[Mnih et~al.(2015)Mnih, Kavukcuoglu, Silver, Rusu, Veness, Bellemare,
  Graves, Riedmiller, Fidjeland, Ostrovski, et~al.]{mnih2015human}
Mnih, V., Kavukcuoglu, K., Silver, D., Rusu, A.~A, Veness, J., Bellemare,
  M.~G., Graves, A., Riedmiller, M., Fidjeland, A.~K., Ostrovski, G., et~al.
\newblock Human-level control through deep reinforcement learning.
\newblock \emph{Nature}, 518\penalty0 (7540):\penalty0 529--533, 2015.

\bibitem[Mnih et~al.(2016)Mnih, Badia, Mirza, Graves, Lillicrap, Harley,
  Silver, and Kavukcuoglu]{mnih2016asynchronous}
Mnih, V., Badia, A.~P., Mirza, M., Graves, A., Lillicrap, T.~P., Harley, T.,
  Silver, D., and Kavukcuoglu, K.
\newblock Asynchronous methods for deep reinforcement learning.
\newblock In \emph{Int. Conf. on Machine Learning}, 2016.

\bibitem[Neumann(2011)]{neumann2011variational}
Neumann, G.
\newblock Variational inference for policy search in changing situations.
\newblock In \emph{Int. Conf. on Machine Learning}, pp.\  817--824, 2011.

\bibitem[O'Donoghue et~al.(2016)O'Donoghue, Munos, Kavukcuoglu, and
  Mnih]{o2016pgq}
O'Donoghue, B., Munos, R., Kavukcuoglu, K., and Mnih, V.
\newblock {PGQ}: Combining policy gradient and {Q}-learning.
\newblock \emph{arXiv preprint arXiv:1611.01626}, 2016.

\bibitem[Otsuka et~al.(2010)Otsuka, Yoshimoto, and Doya]{otsuka2010free}
Otsuka, M., Yoshimoto, J., and Doya, K.
\newblock Free-energy-based reinforcement learning in a partially observable
  environment.
\newblock In \emph{ESANN}, 2010.

\bibitem[Peters et~al.(2010)Peters, M{\"u}lling, and Altun]{peters2010relative}
Peters, J., M{\"u}lling, K., and Altun, Y.
\newblock Relative entropy policy search.
\newblock In \emph{AAAI Conf. on Artificial Intelligence}, pp.\  1607--1612,
  2010.

\bibitem[Rawlik et~al.(2012)Rawlik, Toussaint, and
  Vijayakumar]{rawlik2012stochastic}
Rawlik, K., Toussaint, M., and Vijayakumar, S.
\newblock On stochastic optimal control and reinforcement learning by
  approximate inference.
\newblock \emph{Proceedings of Robotics: Science and Systems VIII}, 2012.

\bibitem[Sallans \& Hinton(2004)Sallans and Hinton]{sallans2004reinforcement}
Sallans, B. and Hinton, G.~E.
\newblock Reinforcement learning with factored states and actions.
\newblock \emph{Journal of Machine Learning Research}, 5\penalty0
  (Aug):\penalty0 1063--1088, 2004.

\bibitem[Schulman et~al.(2015{\natexlab{a}})Schulman, Levine, Abbeel, Jordan,
  and Moritz]{schulman2015trust}
Schulman, J., Levine, S., Abbeel, P., Jordan, M.~I., and Moritz, P.
\newblock Trust region policy optimization.
\newblock In \emph{Int. Conf on Machine Learning}, pp.\  1889--1897,
  2015{\natexlab{a}}.

\bibitem[Schulman et~al.(2015{\natexlab{b}})Schulman, Moritz, Levine, Jordan,
  and Abbeel]{schulman2015high}
Schulman, J., Moritz, P., Levine, S., Jordan, M., and Abbeel, P.
\newblock High-dimensional continuous control using generalized advantage
  estimation.
\newblock \emph{arXiv preprint arXiv:1506.02438}, 2015{\natexlab{b}}.

\bibitem[Schulman et~al.(2017)Schulman, Abbeel, and
  Chen]{schulman2017equivalence}
Schulman, J., Abbeel, P., and Chen, X.
\newblock Equivalence between policy gradients and soft {Q}-learning.
\newblock \emph{arXiv preprint arXiv:1704.06440}, 2017.

\bibitem[Shelhamer et~al.(2016)Shelhamer, Mahmoudieh, Argus, and
  Darrell]{shelhamer2016loss}
Shelhamer, E., Mahmoudieh, P., Argus, M., and Darrell, T.
\newblock Loss is its own reward: Self-supervision for reinforcement learning.
\newblock \emph{arXiv preprint arXiv:1612.07307}, 2016.

\bibitem[Silver et~al.(2014)Silver, Lever, Heess, Degris, Wierstra, and
  Riedmiller]{silver2014deterministic}
Silver, D., Lever, G., Heess, N., Degris, T., Wierstra, D., and Riedmiller, M.
\newblock Deterministic policy gradient algorithms.
\newblock In \emph{Int. Conf on Machine Learning}, 2014.

\bibitem[Silver et~al.(2016)Silver, Huang, Maddison, Guez, Sifre, van~den
  Driessche, Schrittwieser, Antonoglou, Panneershelvam, Lanctot, Dieleman,
  Grewe, Nham, Kalchbrenner, Sutskever, Lillicrap, Leach, Kavukcuoglu, Graepel,
  and Hassabis]{silver2016mastering}
Silver, D., Huang, A., Maddison, C.~J., Guez, A., Sifre, L., van~den Driessche,
  G., Schrittwieser, J., Antonoglou, I., Panneershelvam, V., Lanctot, M.,
  Dieleman, S., Grewe, D., Nham, J., Kalchbrenner, N., Sutskever, I.,
  Lillicrap, T., Leach, M., Kavukcuoglu, K., Graepel, T., and Hassabis, D.
\newblock Mastering the game of go with deep neural networks and tree search.
\newblock \emph{Nature}, 529\penalty0 (7587):\penalty0 484--489, Jan 2016.
\newblock ISSN 0028-0836.
\newblock Article.

\bibitem[Sutton \& Barto(1998)Sutton and Barto]{sutton1998reinforcement}
Sutton, R.~S. and Barto, A.~G.
\newblock \emph{Reinforcement learning: An introduction}, volume~1.
\newblock MIT press Cambridge, 1998.

\bibitem[Thomas(2014)]{thomas2014bias}
Thomas, P.
\newblock Bias in natural actor-critic algorithms.
\newblock In \emph{Int. Conf. on Machine Learning}, pp.\  441--448, 2014.

\bibitem[Todorov(2007)]{todorov2006linearly}
Todorov, E.
\newblock Linearly-solvable {M}arkov decision problems.
\newblock In \emph{Advances in Neural Information Processing Systems}, pp.\
  1369--1376. MIT Press, 2007.

\bibitem[Todorov(2008)]{todorov2008general}
Todorov, E.
\newblock General duality between optimal control and estimation.
\newblock In \emph{IEEE Conf. on Decision and Control}, pp.\  4286--4292. IEEE,
  2008.

\bibitem[Todorov(2009)]{todorov2009compositionality}
Todorov, E.
\newblock Compositionality of optimal control laws.
\newblock In \emph{Advances in Neural Information Processing Systems}, pp.\
  1856--1864, 2009.

\bibitem[Toussaint(2009)]{toussaint2009robot}
Toussaint, M.
\newblock Robot trajectory optimization using approximate inference.
\newblock In \emph{Int. Conf. on Machine Learning}, pp.\  1049--1056. ACM,
  2009.

\bibitem[Uhlenbeck \& Ornstein(1930)Uhlenbeck and
  Ornstein]{uhlenbeck1930theory}
Uhlenbeck, G.~E. and Ornstein, L.~S.
\newblock On the theory of the brownian motion.
\newblock \emph{Physical review}, 36\penalty0 (5):\penalty0 823, 1930.

\bibitem[Wang \& Liu(2016)Wang and Liu]{wang2016learning}
Wang, D. and Liu, Q.
\newblock Learning to draw samples: With application to amortized mle for
  generative adversarial learning.
\newblock \emph{arXiv preprint arXiv:1611.01722}, 2016.

\bibitem[Williams(1992)]{williams1992simple}
Williams, Ronald~J.
\newblock Simple statistical gradient-following algorithms for connectionist
  reinforcement learning.
\newblock \emph{Machine learning}, 8\penalty0 (3-4):\penalty0 229--256, 1992.

\bibitem[Zhao et~al.(2016)Zhao, Mathieu, and LeCun]{zhao2016energy}
Zhao, J., Mathieu, M., and LeCun, Y.
\newblock Energy-based generative adversarial network.
\newblock \emph{arXiv preprint arXiv:1609.03126}, 2016.

\bibitem[Ziebart(2010)]{ziebart2010modeling}
Ziebart, B.~D.
\newblock \emph{Modeling purposeful adaptive behavior with the principle of
  maximum causal entropy}.
\newblock PhD thesis, 2010.

\bibitem[Ziebart et~al.(2008)Ziebart, Maas, Bagnell, and
  Dey]{ziebart2008maximum}
Ziebart, B.~D., Maas, A.~L., Bagnell, J.~A., and Dey, A.~K.
\newblock Maximum entropy inverse reinforcement learning.
\newblock In \emph{AAAI Conference on Artificial Intelligence}, pp.\
  1433--1438, 2008.

\end{thebibliography}
\bibliographystyle{icml2017}

\makeatletter
\def\mathcolor#1#{\@mathcolor{#1}}
\def\@mathcolor#1#2#3{%
  \protect\leavevmode
  \begingroup
    \color#1{#2}#3%
  \endgroup
}
\makeatother

\newcommand{\EE}[2]{\mathbb{E}_{#1}\left[#2\right]}
\newcommand{\pitilde}{\tilde{\pi}}
\newcommand{\pitildered}{\mathcolor{red}{\pitilde}}
\newcommand{\Qpisoft}{Q^{\pi}_{\text{soft}}}
\newcommand{\Qpitildesoft}{Q^{\tilde{\pi}}_{\text{soft}}}
\newcommand{\Qpistarsoft}{Q^{\pi^*}_{\text{soft}}}
\newcommand{\Vpisoft}{V^{\pi}_{\text{soft}}}
\newcommand{\eps}{\varepsilon}

\onecolumn

\appendix
\appendixpage
%\section*{Appendices}
%\addcontentsline{toc}{section}{Appendices}
%\renewcommand{\thesubsection}{\Alph{subsection}}
%\setcounter{subsection}{0}
\section{Policy Improvement Proofs}
\label{app:proofs}

%\subsection{Policy Improvement Proofs}
In this appendix, we present proofs for the theorems that allow us to show that soft Q-learning leads to policy improvement with respect to the maximum entropy objective.
First, we define a slightly more nuanced version of the maximum entropy objective that allows us to incorporate a discount factor. This definition is complicated by the fact that, when using a discount factor for policy gradient methods, we typically do not discount the state distribution, only the rewards. In that sense, discounted policy gradients typically do not optimize the true discounted objective. Instead, they optimize average reward, with the discount serving to reduce variance, as discussed by \citet{thomas2014bias}. However, for the purposes of the derivation, we can define the objective that \emph{is} optimized under a discount factor as
\begin{align*}
\pi_\text{MaxEnt}^* = 
\arg\max_\pi \sum_t \EE{(\st,\at) \sim \rho_\pi}{
\sum_{l=t}^\infty \discount^{l-t} \EE{(\state_l,\action_l)}{ \reward(\st,\at) + \temp \ent(\pi(\sdots|\st)) | \st, \at }
}.
\end{align*}
This objective corresponds to maximizing the discounted expected reward and entropy for future states originating from every state-action tuple $(\st,\at)$ weighted by its probability $\rho_\pi$ under the current policy. Note that this objective still takes into account the entropy of the policy at future states, in contrast to greedy objectives such as Boltzmann exploration or the approach proposed by \citet{o2016pgq}.

We can now derive policy improvement results for soft Q-learning. We start with the definition of the soft Q-value $\Qpisoft$ for any policy $\pi$ as the expectation under $\pi$ of the discounted sum of rewards and entropy :
\begin{equation}
\begin{split}
    \Qpisoft(\state,\action) \triangleq \rz + \E{\tau \sim \pi, \state_0=\state, \action_0=\action}{\sum_{t=1}^{\infty} \gamma^t (r_t + \entropy(\pi(\sdots | \state_t)))}.
\end{split}
\end{equation}
Here, $\tau = (\state_0, \action_0, \state_1, \action_1, \ldots)$ denotes the trajectory originating at $(\state,\action)$. Notice that for convenience, we set the entropy parameter $\alpha$ to 1. The theory can be easily adapted by dividing rewards by $\alpha$. 

The discounted maximum entropy policy objective can now be defined as
\begin{equation}
J(\pi)\triangleq \sum_t \EE{(\st,\at) \sim \rho_\pi}{\Qpisoft(\st,\at) + \temp\ent(\pi(\sdots|\st))}.
\end{equation}
%As we will see later, $\rho_0$ does not affect the form of the maximum entropy policy.

%\subsubsection{The maximum entropy policy}
\subsection{The Maximum Entropy Policy}
\label{app:max_ent_policy}
If the objective function is the expected discounted sum of rewards, the policy improvement theorem \cite{sutton1998reinforcement} describes how policies can be improved monotonically. There is a similar theorem we can derive for the maximum entropy objective:
\begin{theorem}
\label{the:policy_improvement}
(Policy improvement theorem) Given a policy $\policy$, define a new policy $\pitilde$ as
\begin{equation}
\pitilde(\sdots |\state) \propto \exp\left(\Qpisoft(\state,\sdots)\right), \quad \forall\state.
\end{equation}
Assume that throughout our computation, $\Q$ is bounded and $\int \exp(\Q(\state,\action)) \ d\action$ is bounded for any $\state$ (for both $\pi$ and $\pitilde$). Then $\Qpitildesoft(\state,\action) \geq \Qpisoft(\state,\action)\ \forall \state,\action$. 
\end{theorem}
The proof relies on the following observation: if one greedily maximize the sum of entropy and value with one-step look-ahead, then one obtains $\pitilde$ from $\pi$:
\begin{equation}
\begin{split}
\entropy(\pi(\sdots | \state)) + \EE{\action \sim \policy}{\Qpisoft(\state,\action)} \leq \entropy(\pitildered(\sdots | \state)) + \E{\action \sim \pitildered}{\Qpisoft(\state,\action)}.
\end{split}
\end{equation}
The proof is straight-forward by noticing that 
\begin{equation}
\begin{split}
\entropy(\pi(\sdots | \state)) + \E{\action \sim \pi}{\Qpisoft(\state,\action)} = - \kl{\pi(\sdots|\state)}{\pitilde(\sdots|\state)} + \log\int \exp\left(\Qpisoft(\state,\action)\right) \ d\action
\end{split}
\end{equation}
% We write $\EEE_{\tau \sim \pitildered, \sz=\state,\az=\action}$ with a shorthand $\EEE$. 
Then we can show that
\begin{align}
\Qpisoft(\state,\action) &= \E{\state_1}{\rz + \discount(\entropy(\pi(\sdots | \state_1)) + \E{\action_1 \sim \pi}{\Qpisoft(\state_1,\action_1)})} \notag\\
&\leq \E{\state_1}{\rz + \discount(\entropy(\pitildered(\sdots | \state_1)) + \E{\action_1 \sim \pitildered}{\Qpisoft(\state_1,\action_1)})} \notag\\
&= \E{\state_1}{\reward_0 + \discount(\entropy(\pitildered(\sdots | \state_1)) + \reward_1)}  + \discount^2\E{\state_2}{\entropy(\pi(\sdots | \state_2)) + \E{\action_2 \sim \pi}{\Qpisoft(\state_2,\action_2)}} \notag\\
&\leq \E{\state_1}{\rz + \discount(\entropy(\pitildered(\sdots | \state_1)) +\reward_1} + \discount^2\E{\state_2}{\entropy(\pitildered(\sdots | \state_2)) + \E{\action_2 \sim \pitildered}{\Qpisoft(\state_2,\action_2)}} \notag\\
&= \E{\state_1\, \action_2\sim\pitildered, \state_2}{\rz + \discount(\entropy(\pitildered(\sdots | \state_1)) + \reward_1) + \discount^2(\entropy(\pitildered(\sdots | \state_2)) + \reward_2)} + \discount^3\E{\state_3}{\entropy(\pitildered(\sdots | \state_3)) + \E{\action_3 \sim \pitildered}{\Qpisoft(\state_3,\action_3)}} \notag\\
&\ \ \vdots \notag\\
%&\quad (\text{continuing forever}) \notag\\
&\leq \E{\tau \sim \pitildered}{\rz + \sum_{t=1}^{\infty}\discount^t(\entropy(\pitildered(\sdots | \st)) + \rt)} \notag\\
&= \Qpitildesoft(\state,\action).
\end{align}

With \autoref{the:policy_improvement}, we start from an arbitrary policy $\pi_0$ and define the \textit{policy iteration} as
\begin{equation}
\pi_{i+1}(\sdots |\state) \propto \exp\left(\Q^{\pi_i}_{\mathrm{soft}}(\state,\sdots )\right).
\end{equation}
Then $\Q^{\pi_i}_{\text{soft}}(\state,\action)$ improves monotonically. Under certain regularity conditions, $\pi_i$ converges to $\pi_{\infty}$. Obviously, we have $\pi_{\infty}(\action|\state) \propto_\action \exp\left(\Q^{\pi_{\infty}}(\state,\action)\right)$. Since any non-optimal policy can be improved this way, the optimal policy must satisfy this energy-based form. Therefore we have proven \autoref{the:ebm}.

%\subsubsection{Soft Bellman equation and soft value iteration}
\subsection{Soft Bellman Equation and Soft Value Iteration}
\label{app:soft_bellman_equation}
Recall the definition of the soft value function:
\begin{equation}
\Vpisoft(\state)\triangleq \log \int \exp \left(\Qpisoft(\state,\action)\right) \ d\action.
\end{equation}
Suppose $\pi(\action|\state) = \exp\left(\Qpisoft(\state,\action) - \Vpisoft(\state)\right)$. Then we can show that
\begin{align}
    \Qpisoft(\state,\action) &= \reward(\state,\action) + \discount \E{\state' \sim \pdyn}{\entropy(\pi(\sdots | \state')) + \E{\action' \sim \pi (\sdots | \state')}{\Qpisoft(\state',\action')}} \notag\\
    &=\reward(\state,\action) + \discount \E{\state' \sim \pdyn}{\Vpisoft(\state')}.
\end{align}
This completes the proof of \autoref{the:soft_bellman}.

Finally, we show that the soft value iteration operator $\mathcal{T}$, defined as 
\begin{equation}
\mathcal{T}\Q(\state,\action)\triangleq \reward(\state,\action) + \discount \E{\state' \sim \pdyn}{\log \int \exp \Q(\state',\action') \ d\action'},
\end{equation}
is a contraction. Then \autoref{the:soft_q_iteration} follows immediately.

The following proof has also been presented by \citet{fox2015taming}. Define a norm on Q-values as $\|\Q_1 - \Q_2\| \triangleq \max_{\state,\action} |\Q_1(\state,\action) - \Q_2(\state,\action)|$. Suppose $\eps = \|\Q_1 - \Q_2\|$. Then 
\begin{align}
    \log \int \exp (\Q_1(\state',\action')) \ d\action' &\leq \log \int \exp( \Q_2(\state',\action')  + \eps)\ d\action'\notag\\
    &= \log \left(\exp(\eps)\int \exp \Q_2(\state',\action') \ d\action'\right) \notag\\
    &= \eps + \log \int \exp \Q_2(\action',\action') \ d\action'.
\end{align}
Similarly, $\log \int \exp \Q_1(\state',\action') \ d\action' \geq -\eps + \log \int \exp \Q_2(\state',\action') \ d\action'$. Therefore $\|\mathcal{T}\Q_1 - \mathcal{T}\Q_2\| \leq \discount \eps = \discount \|\Q_1 - \Q_2\|$. So $\mathcal{T}$ is a contraction. As a consequence, only one Q-value satisfies the soft Bellman equation, and thus the optimal policy presented in \autoref{the:ebm} is unique.

\section{Connection between Policy Gradient and Q-Learning}
\label{app:pg}
We show that entropy-regularized policy gradient can be viewed as performing soft Q-learning on the maximum-entropy objective. First, suppose that we parametrize a stochastic policy as
\begin{align}
\policy^\pgparams(\at|\st) \triangleq \exp\left(\energy^\pgparams(\st, \at) - \bar\energy^\pgparams(\st)\right),
\label{eq:boltzmann_policy}
\end{align}
where $\energy^\pgparams(\st, \at)$ is an energy function with parameters $\pgparams$, and $\bar\energy^\pgparams(\st) =\log\int_\aspace \exp\energy^\pgparams(\st, \at)d\at$ is the corresponding partition function. This is the most general class of policies, as we can trivially transform any given distribution $p$ into exponential form by defining the energy as $\log p$. We can write an entropy-regularized policy gradient as follows:
\begin{align}
\nabla_\pgparams \pgloss(\pgparams) =  \E{(\st, \at)\sim\rho_{\policy^\pgparams}}{\nabla_\pgparams\log\policy^\pgparams(\at|\st)\left(\hat\Q_{\policy^\pgparams}(\st, \at)+ b^\pgparams(\st)\right)}   + \nabla_\pgparams\E{\st\sim\rho_{\policy^\pgparams}}{\ent(\policy^\pgparams(\sdots | \st))},
\label{eq:pg_ent}
\end{align}
where $\rho_{\policy^\pgparams}(\st, \at)$ is the distribution induced by the policy, $\hat\Q_{\policy^\pgparams}(\st, \at)$ is an empirical estimate of the Q-value of the policy, and $b^\pgparams(\st)$ is a state-dependent baseline that we get to choose. The gradient of the entropy term is given by
\begin{align}
\nabla_\pgparams \ent(\policy^\pgparams) = &-\nabla_\pgparams\E{\st\sim\rho_{\policy^\pgparams}}{\E{\at\sim\policy^\pgparams(\at|\st)}{\log\policy^\pgparams(\at|\st)}}\notag\\
=&-\E{(\st, \at)\sim\rho_{\policy^\pgparams}}{\nabla_\pgparams\log\policy^\pgparams(\at|\st)\log\policy^\pgparams(\at|\st) + \nabla_\pgparams\log\policy^\pgparams(\at|\st)}\notag\\
=&-\E{(\st,\at)\sim\rho_{\policy^\pgparams}}{\nabla_\pgparams\log\policy^\pgparams(\at|\st)\left(1 + \log\policy^\pgparams(\at|\st)\right)},
\end{align}
and after substituting this back into \autoref{eq:pg_ent}, noting \autoref{eq:boltzmann_policy}, and choosing $b^\pgparams(\st) = \bar\energy^\pgparams(\st) + 1$, we arrive at a simple form for the policy gradient:
\begin{align}
%\nabla_\pgparams \pgloss(\pgparams)&=\E{(\st, \at)\sim\rho_{\policy^\pgparams}}{\nabla_\pgparams\log\policy^\pgparams(\at|\st)\left(\hat\Q_{\policy^\pgparams}(\st, \at) - b^\pgparams(\st) - 1 - \log\policy^\pgparams(\at|\st)\right)}\notag\\
%&=\E{(\st, \at)\sim\rho_{\policy^\pgparams}}{\left(\nabla_\pgparams\energy^\pgparams(\st, \at) - \nabla_\pgparams\bar\energy^\pgparams(\st)\right)\left(\Q_{\policy^\pgparams}(\st, \at) - b^\pgparams(\st) - 1 + \bar\energy^\pgparams(\st)- \energy^\pgparams(\st, \at)\right)}\notag\\
&=\E{(\st, \at)\sim\rho_{\policy^\pgparams}}{\left(\nabla_\pgparams\energy^\pgparams(\st, \at) - \nabla_\pgparams\bar\energy^\pgparams(\st)\right)\left(\hat\Q_{\policy^\pgparams}(\st, \at) - \energy^\pgparams(\st, \at)\right)}.
\label{eq:pg}
\end{align}

To show that \autoref{eq:pg} indeed correponds to soft Q-learning update, we consider the Bellman error
\begin{align}
\qloss(\qparams) &= \E{\st\sim q_\st, \at \sim q_\at}{\frac{1}{2}\left(\Qhatsoft^\qparams(\st, \at)- \Qsoft^\qparams(\st, \at)\right)^2},
\end{align}
where $\Qhatsoft^\qparams$ is an empirical estimate of the soft Q-function. There are several valid alternatives for this estimate, but in order to show a connection to policy gradient, we choose a specific form
\begin{align}
\Qhatsoft^\qparams(\st, \at) = \Ahatsoft^\qtargetparams(\st, \at) + \Vsoft^\qparams(\st),
\end{align}
where $\Ahatsoft^\qtargetparams$ is an  empirical soft advantage function that is assumed not to contribute the gradient computation. With this choice, the gradient of the Bellman error becomes
\begin{align}
\nabla_\qparams\qloss(\qparams) &= \E{\st\sim q_\st, \at\sim q_\at}{\left(\nabla_\params \Qsoftparams(\st, \at) - \nabla_\qparams\Vsoftparams(\st)\right)\left(\Ahatsoft^\qtargetparams(\st, \at) + \Vsoft^\qparams(\st,\at) - \Qsoftparams(\st, \at) \right)}\notag\\
&= \E{\st\sim q_\st, \at\sim q_\at}{\left(\nabla_\params \Qsoftparams(\st, \at) - \nabla_\qparams\Vsoftparams(\st)\right)\left(\Qhatsoft^\qparams(\st, \at) - \Qsoft^\qparams(\st, \at)\right)}.
\label{eq:advantage_gradient}
\end{align}
Now, if we choose $\energy^\pgparams(\st, \at) \triangleq \Qsoftparams(\st, \at)$ and $q_\st(\st) q_\at(\at) \triangleq\rho_{\policy^\pgparams}(\st, \at)$, we recover the policy gradient in \autoref{eq:pg}. Note that the choice of using an empirical estimate of the soft advantage rather than soft Q-value makes the target independent of the soft value, and at convergence, $\Qsoftparams$ approximates the soft Q-value up to an additive constant. The resulting policy is still correct, since the Boltzmann distribution in \autoref{eq:boltzmann_policy} is independent of constant shift in the energy function.

\section{Implementation}
\label{app:implementation}

\subsection{Computing the Policy Update}
\label{app:compute_actor_update}

Here we explain in full detail how the policy update direction $\hat\nabla_\pparams\ploss$ in \autoref{alg:soft-q-learning} is computed. We reuse the indices $i,j$ in this section with a different meaning than in the body of the paper for the sake of providing a cleaner presentation.

Expectations appear in amortized SVGD in two places. First, SVGD approximates the optimal descent direction $\phi(\sdots)$ in Equation~\eqref{eq:stein_gradient} with an empirical average over the samples $\ati = f^\pparams(\smpli)$. Similarly, SVGD approximates the expectation in Equation \eqref{eq:actor_gradient} with samples $\attildej = f^\pparams(\smpltildej)$, which can be the same or different from $\ati$. Substituting \eqref{eq:stein_gradient} into \eqref{eq:actor_gradient} and taking the gradient gives the empirical estimate
\begin{equation}
\hat\nabla_\pparams \ploss (\pparams; \st) = \frac{1}{KM}\sum_{j=1}^K\sum_{i=1}^M\bigg( \kernel(\ati, \attildej)\nabla_{\action'} \Qsoft(\st, \action')\big|_{\action' = \ati}\bigg.\notag \bigg.+ \nabla_{\action'}\kernel(\action', \attildej)\big|_{\action'=\ati}\bigg)\nabla_\pparams f^\pparams(\smpltildej;\st).
\label{eq:stein_approx_grad}
\end{equation}
Finally, the update direction $\hat\nabla_\pparams\ploss$ is the average of $\hat\nabla_\pparams \ploss (\pparams; \st)$, where $\st$ is drawn from a mini-batch.

% \subsection{SVGD for Bounded Domains}
% It is noteworthy that the original SVGD method only applies when the distributions are defined on $\mathbb{R}^d$. However, bounded action spaces appear in many practical situations and also throughout the simulation experiments in this paper. Applying SVGD directly on the actions will cause them to saturate at the boundary of the action space instead of approximating the distribution $\exp(Q(s,\sdots))$. A simple trick to avoid this problem is to let the sampling network $f^\pparams(\smpl; \st)$ output unbounded values, while the Q network applies a squashing operator $\sigma(\sdots)$ to those values before feeding them into a standard neural network. For example, if the action at each dimension is bounded by $-1$ and $1$, then we can choose $\sigma(x) = \tanh(x)$. 

\subsection{Computing the Density of Sampled Actions}
Equation \autoref{eq:soft_value_as_expectation} states that the soft value can be computed by sampling from a distribution $q_{\action'}$ and that $q_{\action'}(\sdots) \propto \exp\left(\frac{1}{\alpha}\Qsoft^\pparams(\state, \sdots)\right)$ is optimal. A direct solution is to obtain actions from the sampling network: $\action' = f^\pparams(\smpl'; \state)$. If the samples $\smpl'$ and actions $\action'$ have the same dimension, and if the jacobian matrix $\frac{\partial \action'}{\partial \smpl'}$ is non-singular, then the probability density is
\begin{equation}
q_{\action'}(\action') = p_{\smpl}(\smpl') \frac{1}{\left| \text{det}\left( \frac{\partial \action'}{\partial \smpl'} \right)\right|}.
\end{equation}
In practice, the Jacobian is usually singular at the beginning of training, when the sampler $f^\pparams$ is not fully trained. A simple solution is to begin with uniform action sampling and then switch to $f^\pparams$ later, which is reasonable, since an untrained sampler is unlikely to produce better samples for estimating the partition function anyway.

\section{Experiments}
\label{app:experiments}
% Describe the environments and hyperparameters; present additional results if necessary

\subsection{Hyperparameters}
\label{app: hyperparameters}
Throughout all experiments, we use the following parameters for both DDPG and soft Q-learning. The Q-values are updated using ADAM with learning rate $0.001$. The DDPG policy and soft Q-learning sampling network use ADAM with a learning rate of $0.0001$. The algorithm uses a replay pool of size one million. Training does not start until the replay pool has at least 10,000 samples. Every mini-batch has size $64$. Each training iteration consists of $10000$ time steps, and both the Q-values and policy / sampling network are trained at every time step. All experiments are run for $500$ epochs, except that the multi-goal task uses $100$ epochs and the fine-tuning tasks are trained for $200$ epochs. Both the Q-value and policy / sampling network are neural networks comprised of two hidden layers, with $200$ hidden units at each layer and ReLU nonlinearity. Both DDPG and soft Q-learning use additional OU Noise \cite{uhlenbeck1930theory, lillicrap2015continuous} to improve exploration. The parameters are $\theta = 0.15$ and $\sigma = 0.3$. In addition, we found that updating the target parameters too frequently can destabilize training. Therefore we freeze target parameters for every $1000$ time steps (except for the swimming snake experiment, which freezes for $5000$ epochs), and then copy the current network parameters to the target networks directly ($\tau = 1$).

Soft Q-learning uses $K=M=32$ action samples (see \aref{app:compute_actor_update}) to compute the policy update, except that the multi-goal experiment uses $K=M=100$. The number of additional action samples to compute the soft value is $K_V = 50$. The kernel $\kernel$ is a radial basis function, written as $\kernel(\action,\action') = \exp(-\frac{1}{h}\|\action-\action'\|_2^2)$, where $h = \frac{d}{2\log(M+1)}$, with $d$ equal to the median of pairwise distance of sampled actions $\ati$. Note that the step size $h$ changes dynamically depending on the state $\state$, as suggested in \cite{liu2016stein}. 

The entropy coefficient $\alpha$ is $10$ for multi-goal environment, and $0.1$ for the swimming snake, maze, hallway (pretraining) and U-shaped maze (pretraining) experiments.

All fine-tuning tasks anneal the entropy coefficient $\alpha$ quickly in order to improve performance, since the goal during fine-tuning is to recover a near-deterministic policy on the fine-tuning task. In particular, $\alpha$ is annealed log-linearly to $0.001$ within $20$ epochs of fine-tuning. Moreover, the samples $\smpl$ are fixed to a set $\{\smpl_i\}_{i=1}^{K_{\smpl}}$ and $K_{\smpl}$ is reduced linearly to $1$ within $20$ epochs.

\subsection{Task description}
All tasks have a horizon of $T = 500$, except the multi-goal task, which uses $T=20$. We add an additional termination condition to the quadrupedal 3D robot to discourage it from flipping over.

\subsection{Additional Results}
\label{app:additional-results}

\begin{figure}[h]
    \centering
     \subfigure[DDPG]
    {\includegraphics[width=0.45\columnwidth]{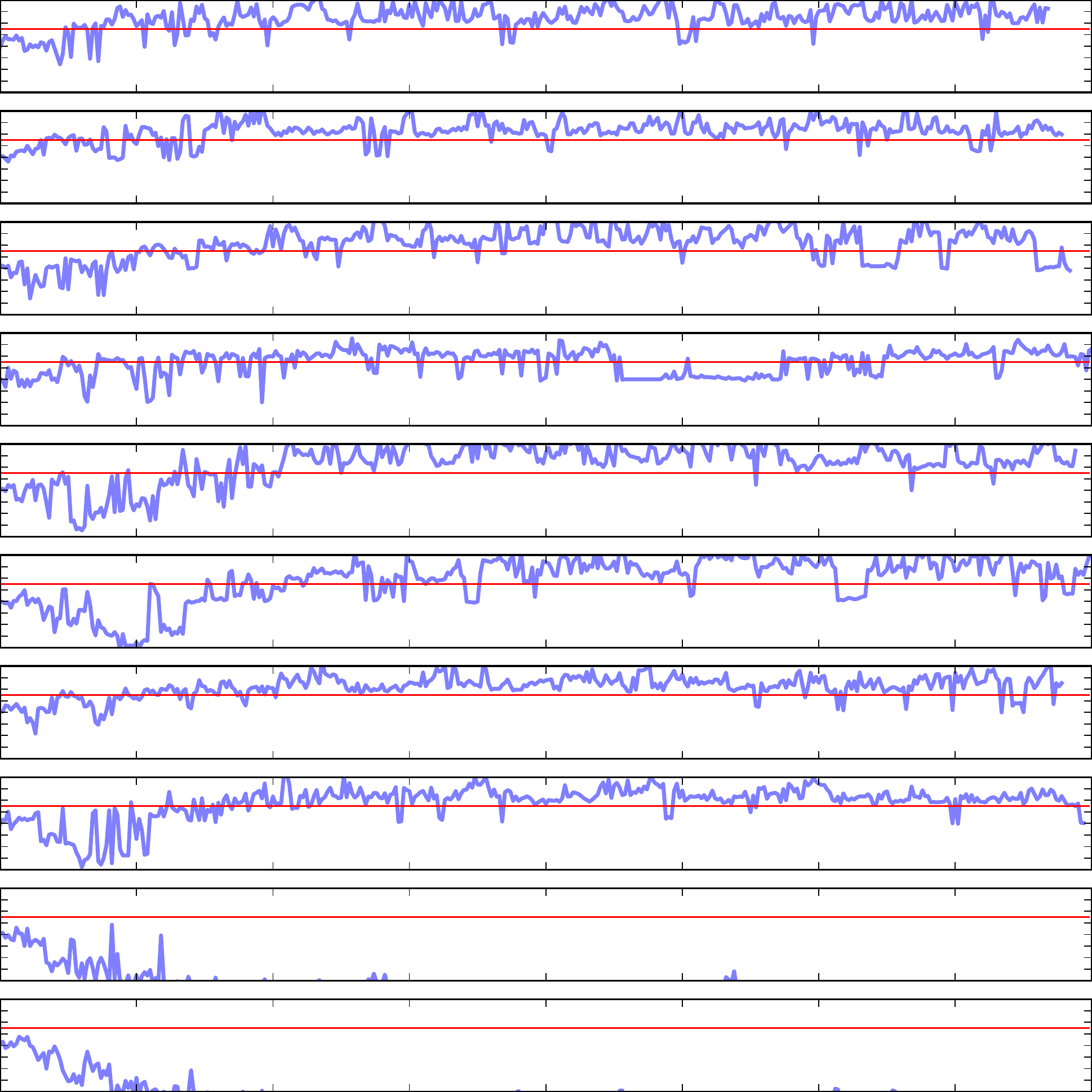}}
    \subfigure[Soft Q-learning]
    {\includegraphics[width=0.45\columnwidth]{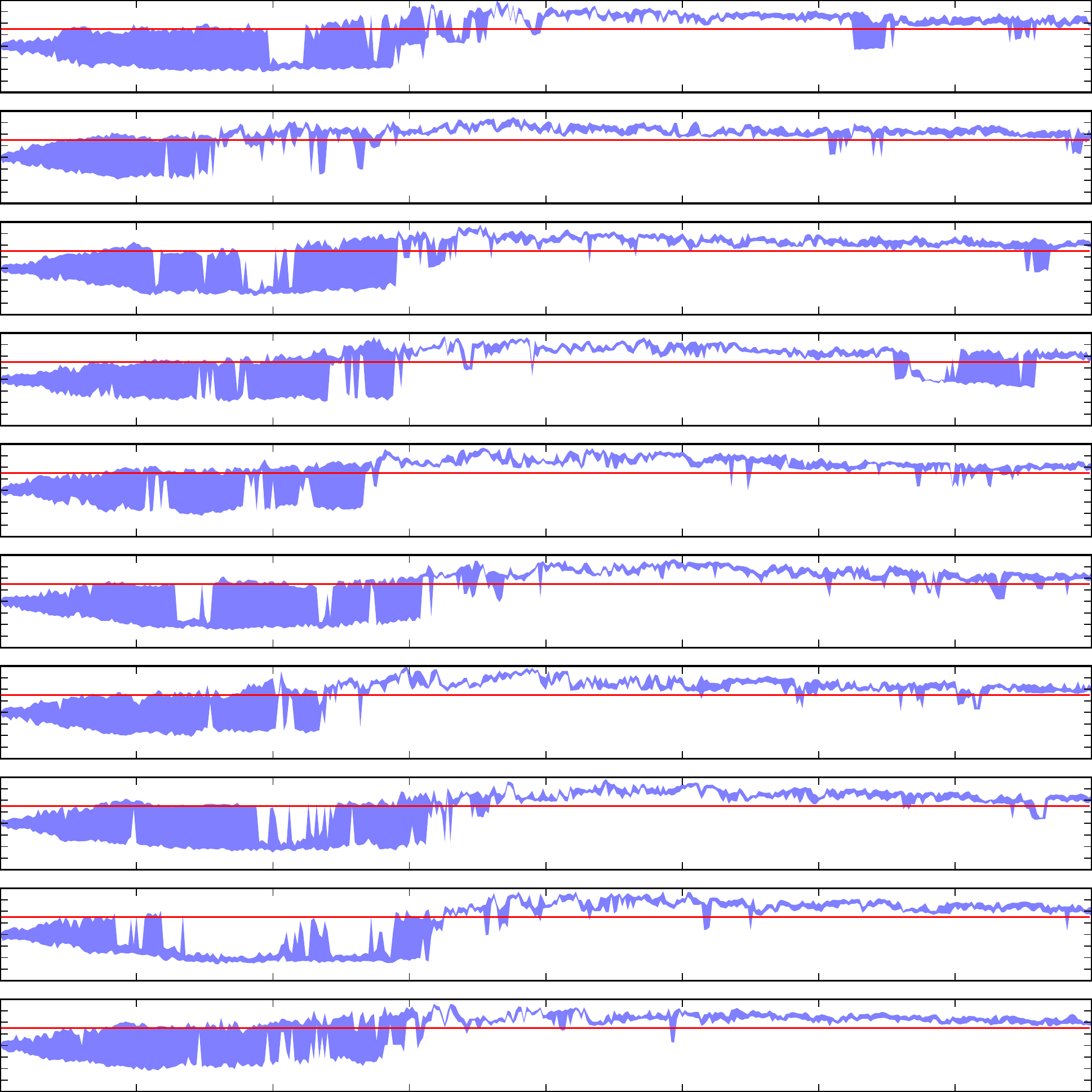}}
    \caption{ 
    Forward swimming distance achieved by each policy. Each row is a policy with a unique random seed. x: training iteration, y: distance (positive: forward, negative: backward). Red line: the ``finish line.'' The blue shaded region is bounded by the maximum and minimum distance (which are equal for DDPG). The plot shows that our method is able to explore equally well in both directions before it commits to the better one.
    \label{fig:swimmer_all_dist_plots}
    }
\end{figure}

%%Haoran.02.21: (0) pick better figures; (1) change the axis limits and font sizes; (2) we could also show the trajs during fine-tuning if space allows
\begin{figure}[ht]
    \centering
    \includegraphics[width=0.8\textwidth]{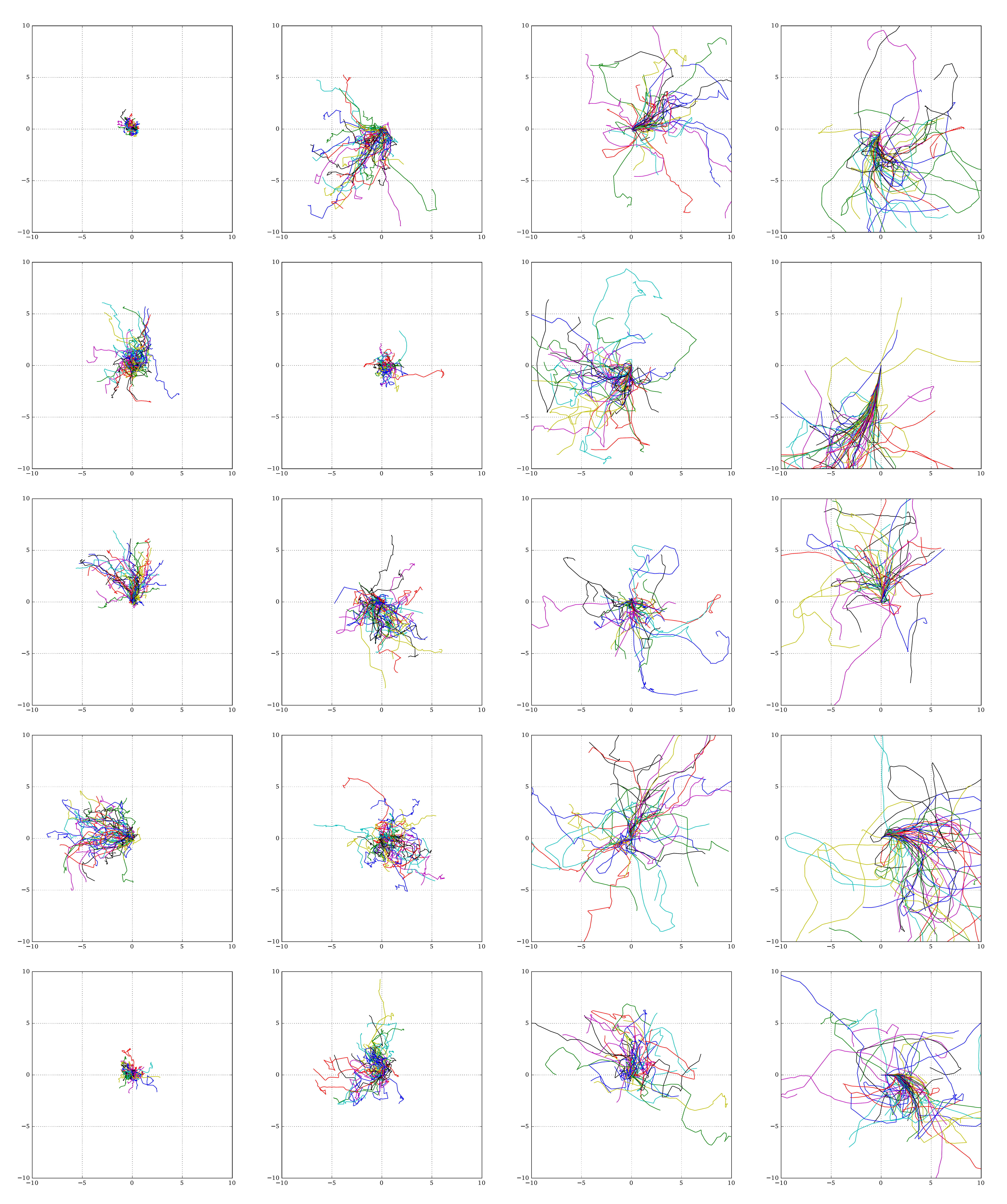}
    \caption{The plot shows trajectories of the quadrupedal robot during maximum entropy pretraining. The robot has diverse behavior and explores multiple directions. The four columns correspond to entropy coefficients $\alpha = 10, 1, 0.1, 0.01$ respectively. Different rows correspond to policies trained with different random seeds. The x and y axes show the x and y coordinates of the center-of-mass.
    As $\alpha$ decreases, the training process focuses more on high rewards, therefore exploring the training ground more extensively. However, low $\alpha$ also tends to produce less diverse behavior. Therefore the trajectories are more concentrated in the fourth column.
    \label{fig:pretrain_trajs}
    }
\end{figure}

\end{document}